\documentclass{article}
\usepackage[utf8]{inputenc}
\usepackage[english]{babel}
\usepackage[top=1in,bottom=1in,left=1in,right=1in]{geometry}

\usepackage{amsmath,amsfonts}
\usepackage{algorithmic}
\usepackage{algorithm}
\usepackage{array}
\usepackage[caption=false,font=normalsize,labelfont=sf,textfont=sf]{subfig}
\usepackage{textcomp}
\usepackage{stfloats}
\usepackage{url}
\usepackage{verbatim}
\usepackage{graphicx}
\usepackage{cite}
\usepackage{bbm,bm,amssymb,mathtools,mathrsfs}
\usepackage{amsthm}
\usepackage{color}
\usepackage{dsfont}
\usepackage{mathrsfs}

\newcommand{\F}{\mathscr F}
\newcommand{\G}{\mathscr G}
\newcommand{\R}{\mathbb R}
\newcommand{\N}{\mathbb N}
\newcommand{\sym}{\mathfrak{S}}
\newcommand{\one}{\mathbbm 1}
\newcommand{\tvec}{\text{vec}}
\newcommand{\sgn}{\mathrm{sgn}}
\newcommand{\ri}{\mathrm{ri}}
\newcommand{\ls}{\mathcal{L}}
\newcommand{\ran}{\R_{\mathrm{an}}}
\newcommand{\ndoc}{N}
\newcommand{\tdoc}{T}
\newcommand{\seqs}{\cup_{t=1}^\infty[\omega]^t}
\newcommand{\seqso}{\cup_{t=0}^\infty[\omega]^t}
\newcommand{\compl}{\mathsf c}
\newcommand{\X}{\mathcal X}
\newcommand{\class}{\mathbb L}
\newcommand{\C}{\mathcal C}
\newcommand{\CC}{\mathscr C}
\DeclareRobustCommand{\looongrightarrow}{%
  \DOTSB\relbar\joinrel\relbar\joinrel\relbar\joinrel\relbar\joinrel\rightarrow
}
\renewcommand{\omega}{\zeta}

\newtheorem{theorem}{Theorem}
\newtheorem{lemma}[theorem]{Lemma}
\newtheorem{remark}[theorem]{Remark}
\newtheorem{definition}[theorem]{Definition}
\newtheorem{proposition}[theorem]{Proposition}
\newtheorem{example}[theorem]{Example}

\newcommand{\cmt}[1]{{\color{black}#1}}

\begin{document}

\title{Next-token Prediction Capacity: General Upper Bounds and a Lower Bound for Transformers}

\author{Liam Madden, Curtis Fox, and Christos Thrampoulidis
\thanks{This work was partially funded by a UBC DSI Postdoctoral Fellowship, NSERC Discovery Grant No. 2021-03677, NSERC ALLRP 581098-22, and a CIFAR AI Catalyst Grant.}
\thanks{Liam Madden and Curtis Fox are with the Department of Computer Science, University of British Columbia, Vancouver, BC, Canada. Liam Madden and Christos Thrampoulidis are with the Department of Electrical and Computer Engineering, University of British Columbia, Vancouver, BC, Canada.}}

\markboth{Journal of \LaTeX\ Class Files,~Vol.~1, No.~2, December~2023}%
{Shell \MakeLowercase{\textit{et al.}}: A Sample Article Using IEEEtran.cls for IEEE Journals}


\maketitle

\begin{abstract}
Given a sequence of tokens, such as words, the task of next-token prediction is to predict the next-token conditional probability distribution. Decoder-only transformers have become effective models for this task, but their properties are still not fully understood. In particular, the largest number of distinct context sequences that a decoder-only transformer can interpolate next-token distributions for has not been established. To fill this gap, we prove upper and lower bounds on this number, which are equal up to a multiplicative constant.
We prove these bounds in the general setting where next-token distributions can be arbitrary as well as the empirical setting where they are calculated from a finite number of document sequences. Our lower bounds are for one-layer multi-head decoder-only transformers and our proofs highlight an important injectivity property satisfied by self-attention. Furthermore, we provide numerical evidence that the minimal number of parameters for memorization is sufficient for being able to train the model to the entropy lower bound.
\end{abstract}


\section{Introduction}
Large language models have become unquestionably effective at certain language tasks, but are they really learning? And if they are, are they learning in ways that we expect; are they learning in a way that can even be tractably unraveled? A researcher with some intuitions about how language works may invent an effective language model, but that does not mean the intuitions are correct. In this paper, we will prove, without recourse to any intuitions, that the now ubiquitous decoder-only transformer model is capable of interpolating arbitrary data sets, in the next-token prediction setting, using the minimal number of parameters required. In fact, it can even do this with an extremely simplified architecture which bares more resemblance to polynomial interpolation than to natural language processing. To get there, we will first give a brief overview of relevant concepts.

Ten years after Bengio et al. \cite{bengio2000neural} introduced neural probabilistic language models,
Tom\'{a}\v{s} Mikolov \cite{mikolov2012phd} constructed a language model that takes a sequence of tokens (e.g. words) as input, embeds the tokens into a vector space using a Skip-gram model, then inputs the sequence of token vectors to a recurrent neural network (RNN). Three years later, the attention mechanism was developed for combining a sequence of vectors into a single vector \cite{bahdanau2015neural}. One year later, self-attention was developed as a modification to attention \cite{cheng2016long}. One year after that, Vaswani et al. \cite{vaswani2017attention} recognized that self-attention made it unnecessary to use a RNN and so replaced the RNN with a feedforward neural network (FNN), introducing the modern transformer model. This gives a brief timeline of the progress leading to modern transformers. But language models have been around for much longer.

75 years ago, Claude Shannon introduced the concept of the entropy of language and used the predictability of language to study it \cite{shannon1948mathematical,shannon1951prediction}. The entropy of a probability mass function $p$ over a finite set $\X$ is the quantity
$-\sum_{x\in\X}p(x)\log p(x)$ and it satisfies certain properties such that it intuitively captures ``how much `choice' is involved in the selection of the event or of how uncertain we are of the outcome'' \cite{shannon1948mathematical}. If we are trying to fit a probability mass function, $q_{\theta}$, parameterized by the vector $\theta\in\R^k$, to $p$, then we can do so by minimizing the cross-entropy of $q_\theta$ relative to $p$, $-\sum_{x\in\X}p(x)\log q_{\theta}(x)$, which is greater than or equal to the entropy of $p$ by Gibb's inequality. In the setting of next-token prediction, we have a vocabulary, $[\omega]=\{1,\ldots,\omega\}$, and we assume that each context sequence $\alpha\in\cup_{t=0}^{T-1}[\omega]^t$ has a corresponding next-token probability mass function $p(\cdot\mid\alpha)$ over $[\omega]$. Here, the cross-entropy of $q_\theta$ relative to $p$ is defined as $-\sum_{t=1}^T\sum_{\alpha\in[\omega]^{t-1}}\sum_{\gamma=1}^\omega p(\alpha,\gamma)\log q_\theta(\gamma\mid \alpha)$.

However, in practice, instead of being given $p$, we are generally given a text corpus of $N$ documents.
\cmt{
For example, the text corpus could consist of books in a library, viewing each book as an individual document, or it could consist of articles on the internet, viewing each article as an individual document.
To analyze the $N$ documents, we view them as sequences of tokens
}
$\beta^1,\ldots,\beta^{N}\in[\omega]^{T}$. In this case, it turns out that the cross-entropy of $q_\theta$ relative to the empirical $\hat{p}$, calculated from the text corpus, is equal to $-\frac{1}{N}\sum_{t=1}^{T}\sum_{j=1}^{N}\log q_{\theta}(\beta_{t}^j\mid \beta_1^j,\ldots,\beta_{t-1}^j)$, and so $\hat{p}$ does not have to be computed. If $\hat{p}(\alpha)$ is nonzero for $n$ $\alpha$'s in $\cup_{t=1}^{T-1}[\omega]^t$,
then we show, in Theorem~\ref{thm:sontag}, that $\approx n\omega$ parameters is necessary for the infimum of the cross-entropy loss to equal the entropy of $\hat{p}$. Then, in Theorem~\ref{thm:full}, we show that when $q_{\theta}$ is a one-layer multi-head decoder-only transformer, $\approx n\omega$ parameters is actually sufficient.

A one-layer multi-head decoder-only transformer model has three sub-layers: the embedding, the (multi-head) self-attention, and the FNN. We will show that this model has optimal interpolation power by first showing that it already does with embedding dimension 1 and only one self-attention head. In this setting, the token embeddings and positional embeddings are scalars. Given an input sequence, the self-attention sub-layer outputs a scalar as well. In Theorem~\ref{thm:injective}, we show that it does so injectively, which allows the model to exploit the interpolation power of the FNN. This leads us to consider simple alternatives to self-attention that satisfy the same injectivity property. As it turns out, taking the weighted average of embedded token sequences, which we call token-averaging, works as well, as we show in Theorem~\ref{thm:avginj}. Thus, replacing self-attention with token-averaging in the transformer model results in another model with optimal memory capacity, as we show in Theorem~\ref{thm:avgmem}.

The memory capacity, the largest $n$ such that a model can interpolate $n$ generic input-output pairs, is a classic topic in learning
theory going back to Cover \cite{cover1965geometrical}, but it is only recently, with the shift in perspective from the classical bias-variance trade-off to the double descent curve of Belkin et al. \cite{belkin2019reconciling}, that going beyond the interpolation threshold has been viewed as a benefit to generalization. Belkin et al. \cite{belkin2019reconciling} argue that when there is an abundance of interpolating solutions, simple algorithms such as gradient descent tend to choose the smoothest interpolating solution, which is a form of Occam's razor and therefore may be at the heart of the ability for highly expressive models to generalize beyond the data sets they are trained on.

\subsection{Results}

We formalize the idea of a probabilistic language space and show that next-token distributions are sufficient to fully specify an underlying stochastic process for language. In particular, Propositions~\ref{thm:firstsecondkind} and~\ref{thm:secondthirdkind} prove that the extension from next-token distributions to a probability measure on the cylindrical $\sigma$-algebra is unique by: (1) showing that next-token distributions correspond to probabilities of thin cylinders, (2) showing that the system of the empty set and thin cylinders is a $\pi$-system, and (3) applying Dynkin's $\pi$-$\lambda$ theorem.

Then we formally define next-token prediction capacity as the memory capacity of a next-token prediction model. We also formally define empirical next-token prediction capacity which applies to the setting where the data comes from a text corpus. We upper bound the next-token prediction capacity for continuously differentiable models in Lemma~\ref{lma:sard}, upper bound the next-token prediction capacity by the empirical next-token prediction capacity for any model in Lemma~\ref{lma:entropy}, and upper bound the empirical next-token prediction capacity for continuously differentiable and $(\ran,\exp)$-definable models in Theorem~\ref{thm:sontag}.

Then we focus our attention on the one-layer multi-head decoder-only transformer model. We lower bound the next-token prediction capacity for this model in Theorem~\ref{thm:full}. Let $k$ be the number of model parameters and $n$ be the number of distinct context sequences. Then our upper and lower bounds are summarized in Table~\ref{table}. The particular constant hidden by $\Omega(\cdot)$ is given in Equation~\eqref{eq:const}.

\begin{table*}[t]
\caption{Next-token Prediction Capacity Bounds for a Transformer with $k$ Parameters}\label{table}
    \centering
    \begin{tabular}{c|c|c}
         \hline
         & General setting & Empirical setting \\
         \hline
         Upper bounds & $\frac{k}{\omega-1}$, Lemma~\ref{lma:sard} & $\left(2+\frac{1}{\omega-1}\right)\frac{k}{\omega-1}+1$, Theorem~\ref{thm:sontag} \\
         \hline
         Lower bounds & $\Omega\left(\frac{k}{\omega-1}\right)$, Theorem~\ref{thm:full} & $\Omega\left(\frac{k}{\omega-1}\right)$, Lemma~\ref{lma:entropy} + Theorem~\ref{thm:full}\\
         \hline
    \end{tabular}
\end{table*}

We prove the next-token prediction capacity lower bound in Theorem~\ref{thm:full} by reducing to the setting with embedding dimension equal to one. In this setting, the self-attention sub-layer outputs scalars to the FNN sub-layer. Given $n$ distinct context sequences, let $x\in\R^n$ denote the $n$ scalar outputs. Then the model has the form $V^\top\psi(wx^\top)$ where $\psi$ is the activation function of the FNN. If $\psi(wx^\top)$ has rank $n$, then we can solve for $V$ given any target $Y\in\R^{\omega\times n}$. Thus, the memory capacity of $V^\top\psi(wx^\top)$ comes down to the rank of $\psi(wx^\top)$.

In Theorem~\ref{thm:polyrank}, we show that if $\psi$ is a polynomial with $k$ monomials and $b\in\R^n$ has entries which are nonzero and distinct, then the rank and Kruskal rank of $\psi(ab^\top)$ are $\min\{m,n,k\}$ for generic $a\in\R^m$. In Theorem~\ref{thm:realrank}, we show that if $\psi$ is real analytic at zero and not a polynomial there, and $b\in\R^n$ has entries which are nonzero and distinct, then the rank and Kruskal rank of $\psi(ab^\top)$ are $\min\{m,n\}$ for generic $a$ in a neighborhood of zero in $\R^m$.

Thus, the rank of $\psi(wx^\top)$ comes down to whether or not the entries of $x$ are nonzero and distinct. To show that they are, we prove, in Theorem~\ref{thm:injective}, that self-attention maps injectively into $\R\backslash\{0\}$. This also leads us to consider token-averaging as a simple alternative to self-attention. In Theorem~\ref{thm:avginj}, we show that token-averaging also maps injectively into $\R\backslash\{0\}$, which leads to an optimal memory capacity lower bound, Theorem~\ref{thm:avgmem}, for the composition of token-averaging and a FNN.

\subsection{Related Work}
\label{sec:relatedwork}

\subsubsection{Probability of language} One way to think about language probabilistically is to assume that for every sequence of $n-1$ tokens (whether letters or words), there is a corresponding next-token distribution. This is equivalent to assuming that there is an underlying Markov process of order $n-1$. This concept is called the $n$-gram model of language and was introduced by Markov in 1913 when he computed the $n$-gram probabilities of sequences of vowels and consonants in Pushkin's \textit{Eugene Onegin} \cite{markov1913}. 35 years later, Shannon analyzed the statistical entropy of the $n$-gram model in two seminal papers, establishing the field of information theory \cite{shannon1948mathematical,shannon1951prediction}.
The concept of statistical entropy goes back to the 1870s when Boltzmann showed that the statistical entropy of microstates in a thermodynamic system is proportional to the thermodynamic entropy of the system.

A more realistic assumption on language is that for every sequence of \textit{any} length there is a corresponding next-token distribution. This idea was introduced in Shannon's 1951 paper when he defined the entropy rate as the limit of the normalized $n$-gram entropies. From the next-token distributions, we can define a function $p$ on sequences of any length using the chain rule of probability, and this $p$ will be a probability mass function when restricted to sequences of any fixed length. We can also define a probability measure on the cylindrical $\sigma$-algebra, and from this probability measure define a stochastic process such that its finite-dimensional distributions are given by $p$. But, unlike with the $n$-gram model, it is not clear that the stochastic process is fully specified by the next-token distributions. To be precise, while the extension of $p$ to a probability measure on the cylindrical $\sigma$-algebra is standard, see e.g. the sequence space subsection of Section 2 in \cite{Billingsley}, the extension proceeds in two steps, first to the cylindrical algebra, then to the cylindrical $\sigma$-algebra, and only the uniqueness of the second step appears to have been noticed, rather than the uniqueness of the full extension.

Another way to think about language probabilistically is to assume that for every sequence of any length there is a corresponding \textit{probability}, i.e. there is a probability mass function over sequences of any length. Booth and Thompson \cite{booth1973applying} call such a language a probability measure language and call its probability mass function a probabilistic word function. But the sequence probabilities cannot come from next-token distributions unless a special end-of-sequence token is included to ensure that the underlying stochastic process terminates with probability one. Keeping the end-of-sequence token but relaxing the assumption that it appears in every sequence with probability one, Du et al. \cite{du2023measure} use the extension from the previous paragraph to induce a probability measure over finite and infinite sequences simultaneously. This is an interesting direction, but is conceptually distinct from the perspective taken in the present paper.

\subsubsection{Memory capacity}
The memory capacity---also known as finite sample expressivity, memorization capacity, storage capacity, or, simply, capacity---of a machine learning model with $k$ parameters is the largest $n$ such that it can interpolate $n$ generic input-output pairs, where by generic we mean that the set of exceptions lies on the zero set of a nontrivial real analytic function and is, therefore, measure zero and closed \cite[Corollary 10]{gunning1965analytic}. The idea of memory capacity goes back to Cover \cite{cover1965geometrical} who considered the separating capacities of families of surfaces. Later, Baum \cite{baum1988multilayer} proved that a two-layer FNN with Heaviside activation has memory capacity at least $\approx k$ where outputs are in $\{\pm 1\}$.
Sakurai \cite{sakurai1992networks} extended this to three-layer FNNs. Huang and Huang \cite{huang1991bounds} proved it for two-layer FNNs with sigmoid activation and outputs in $\R$. Yamasaki \cite{yamasaki1993lower} sketched a proof for $L$-layer FNNs with sigmoid activation. Huang \cite{huang2003learning} proved it for three-layer FNNs with sigmoid activation. Yun et al. \cite{yun2019small} proved it for two-layer FNNs with ReLU activation and outputs in $\{\pm 1\}$ (their Corollary 4.2), and for three-layer FNNs with ReLU activation and outputs in $\R$ (their Theorem 3.1). Bubeck et al. \cite{bubeck2020network} proved it for two-layer FNNs with ReLU activation and outputs in $\R$. Madden and Thrampoulidis \cite{madden2024memory} proved it for two-layer FNNs with general activations (excluding only low degree polynomials and low degree splines) and outputs in $\R$. Madden \cite{madden2024three} proved it for three-layer FNNs with general activations (excluding only polynomials and splines) and outputs in $\R$.

While all of the papers mentioned so far have been about the memory capacity of FNNs, Kim et al. \cite{kim2023provable}, Kajitsuka and Sato \cite{kajitsuka2024are}, and Mahdavi et al. \cite{mahdavi2024memorization} recently proved memory capacity lower bounds for transformers. However, none of these results apply to the next-token prediction setting where inputs are context sequences in $\seqs$ and outputs are next-token distributions in $\Delta^{\omega-1}$, the unit simplex in $\R^{\omega}$. We will show that this is the case for each of the three papers, but first let us clarify two points about the inputs and outputs. Regarding inputs, if $\beta=[\beta_1|\cdots|\beta_T]\in\R^{d\times T}$, then $\beta_1,[\beta_1|\beta_2],[\beta_1|\beta_2|\beta_3],\ldots$ are its unique contexts. More generally, if $\beta^1,\ldots,\beta^N\in\R^{d\times T}$, then $\{[\beta_1^j|\cdots|\beta_t^j]\mid j\in[N],t<T\}$ are the unique contexts of the list. If we use $n$ to denote the number of unique contexts, then it is possible for $N$ to go to infinity while $n$ remains fixed. Thus, the memory capacity for next-token prediction should be in terms of $n$, not $N$. Moving on to outputs, context sequences do not have unique next-tokens, they have next-token distributions, e.g. ``I like'' may be followed by ``dogs'' or ``cats.'' In other words, the next-token distributions are not necessarily one-hot. Thus, for next-token prediction, outputs should be in $\R^\omega$, not $\R^d$. Given these two subtleties, we use next-token prediction capacity to refer to the memory capacity of next-token prediction as opposed to the other settings considered in \cite{kim2023provable}, \cite{kajitsuka2024are}, and \cite{mahdavi2024memorization}.

Kim et al. \cite{kim2023provable} and Kajitsuka and Sato \cite{kajitsuka2024are} proved memory capacity results in the sequence-to-sequence setting, where inputs are document sequences of token embeddings in $\R^{d\times T}$ and outputs are in $[C]^{1\times \tdoc}$ or $[C]^{d\times \tdoc}$ for some $C\in\N$. This does not apply to the next-token prediction setting for two reasons. First, the number of document sequences, $N$, could go to infinity while the number of unique contexts, $n$, remains fixed, thus the number of parameters required to apply the result, $\Omega(\sqrt{NT})$ in \cite{kim2023provable} and $\Omega(NT)$ in \cite{kajitsuka2024are}, is vacuous in the next-token prediction setting. Second, the outputs are not in $\R^{\omega\times \tdoc}$. Remark 3.2 of \cite{kim2023provable} explains how the codomain can be extended to $\R^{\ell\times\tdoc}$ for any $\ell\in\N$ by fixing the precision level, $\epsilon>0$, but the new required number of parameters is $\Omega(\ell\log(\ell/\epsilon)\sqrt{\ndoc\tdoc})$, which goes to infinity as $\epsilon$ goes to zero. Furthermore, both of these results assume that the embedded tokens are sufficiently separated while we do not make any assumptions on inputs.

Mahdavi et al. \cite{mahdavi2024memorization} proved a memory capacity result for transformers in the setting where inputs are context sequences of token embeddings in $\R^{d\times \tau}$ and each output is a prediction of the next-token embedding in $\R^{d}$, showing that $\approx nd$ parameters are sufficient to interpolate $n$ input-output pairs, which is optimal. Interestingly, they analyzed a model without a FNN sub-layer and instead leveraged the number of self-attention heads to prove their result. However, besides only applying to outputs in $\R^d$, rather than $\R^{\omega}$, their result requires that all context sequences have the same length. Thus, their result does not apply to the unique contexts $\beta_1,[\beta_1|\beta_2],[\beta_1|\beta_2|\beta_3],\ldots$ of a document sequence $\beta\in\R^{d\times T}$, only to contexts $\alpha\in\R^{d\times \tau}$ for some fixed $\tau$.

\subsubsection{Optimization of transformers} We have reviewed the works most closely related to our memory capacity results, but there are many relevant studies exploring the optimization of one-layer transformers as well, such as \cite{tian2023scan,chen2024provably,tarzanagh2023transformers,bartlett-in-context,deora2024on}. However, most do not address the next-token prediction scenario we study here. The exceptions, such as \cite{makkuva2024attention,li2024mechanics,tian2023scan,tian2023joma,ildiz2024self}, rely on restrictive assumptions about input data and transformer dimensions, such as embedding size, to achieve minimal empirical risk. Additionally, optimization and learnability in the next-token prediction setting, albeit limited to linear models, have been recently studied in \cite{malach2023auto,thrampoulidis2024implicit}. Finally, the effectiveness of transformers for particular tasks has been explored in papers such as \cite{sanford2023represent,wang2024transformers}.

\subsubsection{Miscellaneous results}
Our upper bound, Theorem~\ref{thm:sontag}, can be seen as a generalization of Theorem 1 of Sontag \cite{sontag1997shattering} from input-output pairs in $\R^d\times \R$ to input-output pairs in $\seqs\times\Delta^{\omega-1}$. However, our proof is quite different because $\seqs$ is countable.
As for Theorems~\ref{thm:polyrank} and~\ref{thm:realrank}, these are classical results when $a$ and $b$ are both generic, but to the best of our knowledge, they are novel results in the setting where $b\in\R^n$ has entries which are nonzero and distinct. However, closely related results have been shown before. Tamura \cite{tamura1991capabilities} proved a version of Theorem~\ref{thm:realrank} when $\psi$ is the logistic function and restated the result in \cite{tamura1997capabilities}. Then Huang and Babri \cite{huang1998upper} proved it more generally when $\psi$ is bounded, nonlinear, and has a limit at either positive or negative infinity. Their result does not require $\psi$ to be real analytic at a point and not a polynomial there, but our result does not require $\psi$ to be bounded or to have limits. Thus, the two results are complementary. Interestingly, their proof sends the inputs of $\psi$ off to positive or negative infinity, while our proof only uses inputs in the interval of convergence of the point where $\psi$ is real analytic.

\subsection{Organization}

In Section~\ref{sec:prelims} we present notations, conventions, and prerequisites.
In Section~\ref{sec:probability} we carefully delineate the assumptions made to start talking about next-token prediction. Section~\ref{sec:probability} serves to ground the definition of next-token prediction capacity, but the reader that is already willing to accept the definition can safely skip to Section~\ref{sec:nexttoken}.
In Section~\ref{sec:nexttoken} we define next-token prediction capacity and prove the upper bound Theorem~\ref{thm:sontag}.
In Section~\ref{sec:transformer} we define the one-layer multi-head decoder-only transformer model.
In Section~\ref{sec:memory} we prove the injectivity result, the rank results, and the next-token prediction lower bound.
In Section~\ref{sec:experiment} we present a numerical experiment demonstrating that not only is the memory capacity on the same order as the number of parameters divided by the vocabulary size, but the transformer can already be trained to the entropy lower bound in this regime.
In Section~\ref{sec:tokenavg} we introduce the token-averaged FNN model and prove a next-token prediction capacity lower bound for it.
In Section~\ref{sec:conclusion}, we conclude.

\section{Preliminaries}
\label{sec:prelims}

We will use the following conventions: empty sums are zero, empty products are one, and we consider the empty sequence as an object, denoted $(~)$. Given a finite nonempty set $\X$, we define $\X^0=\{(~)\}$. Also, we denote the concatenation of two finite sequences $x$ and $y$ by $(x,y)$ and we denote the length of a sequence $x$ by $|x|$. Finally, we skip terms in a sum that do not exist, such as $0\log 0$.

We will use the following notation: $[n]$ denotes $\{1,\ldots,n\}$, $\Delta^{n-1}=\{x\in\R_{+}^n\mid \|x\|_1=1\}$ denotes the unit simplex in $\R^n$, $\ri\Delta^{n-1}=\{x\in\R_{++}^n\mid \|x\|_1=1\}$ denotes the relative interior of $\Delta^{n-1}$, $\binom{A}{n}$ denotes $\{B\subset A\mid |B|=n\}$, $\tvec$ denotes the column-wise vectorize operation, $e_k\in\R^n$ denotes the $k$th coordinate vector, $\one_n$ denotes the vector of ones in $\R^n$, $I\in\R^{n\times n}$ denotes the identity matrix, $\delta$ denotes the Kronecker delta function, $a^{(k)}$ indicates that the exponent $k$ is applied to the vector $a$ element-wise, $\sym_n$ denotes the symmetric group of degree $n$, $\exp$ denotes the exponential function, and $\phi$ denotes the softmax function,
\begin{align*}
    \phi:\R^n\to\R^n:x\mapsto \frac{\exp(x)}{\sum_{i=1}^n\exp(x_i)}.
\end{align*}

We also use the following conventions for matrices. If $a_i\in\R~\forall i\in[m]$, then $[a_i]_{i\in[m]}$ denotes $(a_1,\ldots,a_m)\in\R^m$. If $a_j\in\R^m~\forall j\in[n]$, then $[a_j]_{j\in[n]}$ denotes $[a_1|\cdots|a_n]\in\R^{m\times n}$. Going the other way, if $a\in\R^m$, then $a_i$ denotes element $i$ of $a$, and if $A\in\R^{m\times n}$, then $a_{i,j}$ denotes element $(i,j)$ of $A$. We will also index vectors and matrices using brackets. If $a\in\R^m$ and $\alpha\in[m]^k$, then $a[\alpha]$ denotes $[a_{\alpha_i}]_{i\in[k]}$. If $A\in\R^{m\times n}$, $\alpha\in[m]^k$, and $\beta\in[n]^{\ell}$, then $A[\alpha,\beta]$ denotes $[[a_{\alpha_i,\beta_j}]_{i\in[k]}]_{j\in[\ell]}$. In this context, a colon denotes $(1,2,\ldots)$, $i:j$ denotes $(i,\ldots,j)$, and $-1$ denotes the last index.

In the realm of real analytic functions, we use Taylor's theorem as as well as the fact that the zero set of a nontrivial real analytic function is measure zero and closed \cite[Corollary 10]{gunning1965analytic}. We say a property holds ``generically'' if the set of exceptions lies on the zero set of a nontrivial real analytic function. Note that $\phi$ is real analytic since the numerator and denominator are real analytic with the denominator being nowhere zero. Moreover, the entire self-attention sub-layer of the transformer model, which will be described in detail below, is real analytic since it is a composition of real analytic functions. Furthermore, if $\psi$ is real analytic at a point, then the entire transformer parameter mapping is real analytic at a point. Similarly, if $\psi$ is continuously differentiable, then the entire transformer parameter mapping is continuously differentiable, in which case we can apply Sard's theorem \cite{sard1941measure}.

\cmt{
\begin{lemma}[Sard's theorem]
\label{lma:sard1941}
Let $k\ge 1$. Let $M$ and $N$ be $C^k$ manifolds of dimension $m$ and $n$ respectively. Let $F:M\to N$ be $C^k$. If $m\le n$ or $m\le n+k-1$, then the set of critical values of $F$ has measure zero in $N$.
\end{lemma}
}

A critical point is a point where the differential is not surjective and a critical value is a value with level set (preimage/fiber) containing at least one critical point. \cmt{For example, if $m<n$, then every point is a critical point, hence the image of $F$ has measure zero in $N$.}

We also use the model-theoretic structure $(\ran,\exp)$ \cite{dries1994elementary,dries1994real,dries1998tame}. Polynomials, restricted analytic functions, and exp are the basic $(\ran,\exp)$-definable functions, but compositions, quotients (if the denominator is nowhere zero), inverses (of injective functions), and limits of $(\ran,\exp)$-definable functions are $(\ran,\exp)$-definable as well \cite[Lma. 2.3]{dries1998tame}. From this it can be seen that $\phi$, tanh, arctan, erf, and GELU are $(\ran,\exp)$-definable. Furthermore, if $\psi$ is $(\ran,\exp)$-definable, then the entire transformer parameter mapping is $(\ran,\exp)$-definable, and so its image either has nonempty interior or is nowhere dense \cite[Corollary B.4]{sontag1996critical}.

In the realm of linear algebra, we use the notion of Kruskal rank \cite{kruskal1977three}, as well as the Cauchy-Binet formula \cite[Sec. I.2.4]{gantmacher1960matrices} and the Leibniz determinant formula \cite[Def. 10.33]{axler2015linear}. In the realm of number theory, we use that $\exp(1)$ is transcendental \cite{hermite1873}. Finally, we also use basic concepts from probability theory to which we refer the reader to \cite{Billingsley} as a reference.

\cmt{
For variables and functions, we will use the notation that is most appropriate to the section. However, there are some notations that are used consistently throughout the paper, so we will summarize those here. We use $\omega$ to denote the number of tokens in the vocabulary. We use $\alpha,\beta,\gamma$ as example sequences from $\cup_{t=0}^\infty[\omega]^t$. We use $N$ to denote the number of document sequences and $n$ to denote the number of distinct context sequences. We use $\theta$ to denote the concatenation of model parameters and $h_\theta$ to denote the model. We use $\phi$ to denote the softmax function and $\psi$ to denote the activation function. We use $Z$ for the embedding matrix, $d$ for the embedding dimension, $U$ for the position matrix, $m_0$ for the number of self-attention heads, $W$ and $V$ for self-attention and FNN matrices, and $m$ for the FNN width.
}

\section{Probabilistic Language Space}
\label{sec:probability}

We start by abstracting notions about how language works into mathematical axioms. Specifically, given a sequence of words, we assume that the next word is governed by a probability distribution. Of course, this is not how language really works---we respond to another person's body language in conversation and the author writes about flowers they see on the side of the road---but it is a useful abstraction for certain tasks such as summarizing material online. We will axiomatize the concept of a probabilistic language space in three different ways and show that they are equivalent.
Note that we assume neither stationarity nor ergodicity.
Refer to Section~\ref{sec:relatedwork} for the context of this section's results within the related literature on the probability of language, tracing back to the seminal papers of Markov and Shannon. Also, note that, while this section serves to ground the concept of next-token prediction capacity, the reader who wishes to take the definition of next-token prediction capacity at face value can safely skip to the next section.

\begin{definition}
\label{def:p1}
A \textbf{probabilistic language space of the first kind} is a triple $(\X,A,p)$ that satisfies the following axioms:
\begin{enumerate}
    \item $\X$ is a finite nonempty set;
    \item $\X^0\subset A\subset \cup_{t=0}^\infty\X^t$;
    \item $p:\X\times A\to[0,1]$;
    \item $\forall x\in\cup_{t=1}^\infty\X^t$, $x\in A\iff p(x_t\mid x_1\ldots,x_{t-1})\neq 0$;
    \item $\forall x\in A$, $\sum_{y\in\X}p(y\mid x)=1$.
\end{enumerate}
\end{definition}

\begin{definition}
\label{def:p2}
A \textbf{probabilistic language space of the second kind} is a couple $(\X,p)$ that satisfies the following axioms:
\begin{enumerate}
    \item $\X$ is a finite nonempty set;
    \item $p:\cup_{t=1}^\infty \X^t\to[0,1]$;
    \item $\sum_{x\in\X}p(x)=1$;
    \item $\forall x\in\cup_{t=1}^\infty \X^t$, $\sum_{y\in\X}p(x,y)=p(x)$.
\end{enumerate}
\end{definition}

\begin{definition}
\label{def:p3}
A \textbf{probabilistic language space of the third kind} is a couple $(\X,P)$ that satisfies the following axioms:
\begin{enumerate}
    \item $\X$ is a finite nonempty set;
    \item $P$ is a probability measure on the cylindrical $\sigma$-algebra of $\X^\N$;
    \item $\sum_{x\in\X}P(\C(\{x\}))=1$;
    \item $\forall x\in\cup_{t=1}^\infty \X^t$, $\sum_{y\in\X}P(\C(\{(x,y)\}))=P(\C(\{x\}))$;
\end{enumerate}
where, for all $t\in\N$ and all $A\subset \X^t$, $\C(A)\coloneqq\{x\in\X^\N\mid (x_1,\ldots,x_t)\in A\}$.
\end{definition}

Axiom 5 of the first definition says that $p(\cdot\mid x)$ is a conditional probability mass function whenever it is defined. Axiom 4 says that $p(\cdot\mid x)$ is defined for all $x$ except those which have zero probability by the chain rule. Definition~\ref{def:p2} essentially applies the chain rule to the conditional probability mass functions from Definition~\ref{def:p1}. Note that for all $t\ge 2$, $\sum_{x\in\X^t}p(x)=\sum_{y\in\X^{t-1}}\sum_{z\in\X}p(y,z)=\sum_{y\in\X^{t-1}}p(y)$, so, by induction, $\sum_{x\in\X^t}p(x)=1$ for all $t\in\N$. In other words, $p$ is a probability mass function on $\X^t$ for every $t\in\N$, but also is compatible with being decomposed into the conditional probability mass functions of Definition~\ref{def:p1}.

The third definition requires greater explanation. Subsets of $\X^\N$ of the form $\C(A)$ for some $A\subset \X^t$ are called cylinders of rank $t$. Subsets of $\X^\N$ of the form $\C(\{x\})$ for some $x\in \cup_{t=1}^\infty\X^t$ are called thin cylinders. Let $\CC_0$ be the class of cylinders. Then $\CC_0$ is an algebra, as shown after Eq.~(2.17) in \cite{Billingsley}. The $\sigma$-algebra generated by it, $\CC$, is called the cylindrical $\sigma$-algebra of $\X^\N$. Axioms 3 and 4 of Definition~\ref{def:p3} ensure that $P$ can be reduced to a $p$ as in Definition~\ref{def:p2}.

It is quite easy to show that the first two definitions are equivalent, but distinguishing them so formally is useful because it makes the connection between the last two definitions more clear. Normally, $p$ is extended to $P$ without proving that there is actually a bijection between the set of $p$ such that $(\X,p)$ is a probabilistic language space of the second kind and the set of $P$ such that $(\X,P)$ is a probabilistic language space of the third kind. We will rigorously prove this after proving the equivalence of the first two definitions.

\begin{proposition}
\label{thm:firstsecondkind}
Let $\X$ be a finite nonempty set.
Define $\class_1$ as the set of all $p:\X\times A\to[0,1]$ such that $(\X,A,p)$ is a probabilistic language space of the first kind.
Define $\class_2$ as the set of all $p:\cup_{t=1}^\infty\X^t\to[0,1]$ such that $(\X,p)$ is a probabilistic language space of the second kind.
Then there is a bijection $\Phi_{12}:\class_1\to\class_2$ such that for all $p:\X\times A\to[0,1]$ in $\class_1$,
\begin{align*}
    \Phi_{12}(p)(x,y)=p(y\mid x)\Phi_{12}(p)(x)~\forall (y,x)\in\X\times A.
\end{align*}
\end{proposition}
\begin{proof}
Define $\Phi_{12}:\class_1\to \class_2$ in the following way. Let $p:\X\times A\to[0,1]$ be in $\class_1$. Extend $p$ to all of $\X\times \cup_{t=0}^\infty \X^t$ by setting $p(y\mid x)=0$ for all $(y,x)\in\X\times A^\compl$. Then, for all $x\in\cup_{t=1}^\infty\X^t$, define $q(x)=\prod_{t=1}^{|x|}p(x_t\mid x_1,\ldots,x_{t-1})$. Then $(\X,q)$ satisfies the first three axioms of Definition~\ref{def:p2} by construction. To see that it satisfies the fourth axiom as well, let $x\in\cup_{t=1}^\infty\X^t$ and observe
\begin{align*}
    \sum_{y\in\X}q(x,y)=q(x)\sum_{y\in\X}p(y\mid x)=q(x).
\end{align*}
Thus, $q\in\class_2$. Define $\Phi_{12}(p)=q$.

We will show that $\Phi_{12}$ is a bijection by constructing its inverse. Define $\Phi_{21}:\class_2\to\class_1$ in the following way. Let $q\in\class_2$. Extend $q$ to all of $\cup_{t=0}^\infty \X^t$ by setting $q(~)=1$. Define $A=\X^0\cup \{x\in\cup_{t=1}^\infty\X^t\mid q(x)\neq 0\}$ and $p:\X\times A\to[0,1]:(y,x)\mapsto q(x,y)/q(x)$. Then $p$ is well-defined and satisfies the first four axioms of Definition~\ref{def:p1} by construction. To see that it satisfies the fifth axiom, we will use induction. First, $\sum_{y\in\X}p(y\mid~)=\sum_{y\in\X}q(y)=1$. Now let $t\in\N$ and suppose $\sum_{y\in \X}p(y\mid x)=1$ for all $x\in A$ such that $|x|<t$. Then let $x\in A$ such that $|x|=t$ and observe
\begin{align*}
    \sum_{y\in\X}p(y\mid x)=\sum_{y\in\X}\frac{q(x,y)}{q(x)}=1.
\end{align*}
Thus, $p\in\class_1$. Define $\Phi_{21}(q)=p$. Clearly, $\Phi_{21}$ is the inverse of $\Phi_{12}$.
\end{proof}

\begin{proposition}
\label{thm:secondthirdkind}
Let $\X$ be a finite nonempty set.
Define $\class_2$ as the set of all $p:\cup_{t=1}^\infty\X^t\to[0,1]$ such that $(\X,p)$ is a probabilistic language space of the second kind.
Define $\class_3$ as the set of all probability measures on $(\X^\N,\CC)$ such that $(\X,P)$ is a probabilistic language space of the third kind.
Define $\Phi_{32}:\class_3\to\class_2:P\mapsto P(\C(\{\cdot\}))$.
Then $\Phi_{32}$ is a bijection.
\end{proposition}
\begin{proof}
First, $\Phi_{32}$ is well-defined by construction. We will show that it is a bijection by constructing its inverse. Define $\Phi_{23}:\class_2\to\class_3$ in the following way. Let $p\in\class_2$. Define $P_0(\emptyset)=0$. For each nonempty cylinder $\C(A)$, define
\begin{align*}
    P_0(\C(A)) = \sum_{x\in A}p(x).
\end{align*}
In order to show that $P_0$ is well-defined, we have to show that the right-hand side is equal for every representation of the cylinder. We will prove this by induction. First, note that a cylinder has at most one representation of each rank and that if $A\subset \X^t$ represents it, then $\{(x,y)\mid x\in A,y\in\X\}\subset \X^{t+1}$ also represents it. Thus, every cylinder has a minimal rank $t_0$. Suppose the right-hand side is equal for every rank in $\{t_0,\ldots,t\}$. Then we will show that it is equal for rank $t+1$ as well. Observe,
\begin{align*}
    \sum_{x\in A}\sum_{y\in\X}p(x,y)=\sum_{x\in A}p(x).
\end{align*}
Thus, by induction, $P_0$ is well-defined. Next, we will show that $P_0$ is a probability measure on $\CC_0$. We already have $P_0(\emptyset)=0$. To get $P_0(\Omega)=1$ we can use $\Omega=\C(\X)$ and $\sum_{x\in\X}p(x)=1$. So the only thing left to show is that $P_0$ is countably additive. But, by Theorem 2.3 of \cite{Billingsley}, the class of cylinders has the special property that finitely additive probability measures on it are countably additive as well. Thus, we only need to show that $P_0$ is finitely additive. Let $\C(A)$ and $\C(B)$ be disjoint. Since every cylinder of rank $t_0$ is also a cylinder of rank $t$ for all $t\ge t_0$, we can choose $A,B$ to be subsets from the same $\X^t$. Thus, $A$ and $B$ are disjoint and $\C(A)\cup\C(B)=\C(A\cup B)$. So,
\begin{align*}
    P_0(\C(A\cup B)) &= \sum_{x\in A\cup B}p(x)\\
    &=\sum_{x\in A}p(x)+\sum_{x\in B}p(x)\\
    &=P_0(\C(A))+P_0(\C(B)),
\end{align*}
proving that $P_0$ is finitely additive, hence countably additive, hence a probability measure on $\CC_0$. So, by the Extension Theorem \cite{Fréchet}, $P_0$ has a unique extension to a probability measure $P$ on $\CC$. Define $\Phi_{23}(p)=P$.

It is obvious that $(\Phi_{32}\circ \Phi_{23})(p)=p$ for all $p\in\class_2$. To prove that $(\Phi_{23}\circ \Phi_{32})(P)=P$ for all $P\in\class_3$, on the other hand, is quite subtle. Let $P_1\in\class_3$. Define $p=\Phi_{32}(P_1)$ and $P_2=\Phi_{23}(p)$. While the uniqueness part of the Extension Theorem allowed us to define $\Phi_{23}$, it does not guarantee that $P_1=P_2$. To prove the latter, we will use 
Theorem 3.3 of \cite{Billingsley}, a consequence of Dynkin's $\pi$-$\lambda$ theorem, which can be used to prove the uniqueness part of the Extension Theorem as well. Let $\G$ be the class of the empty set and thin cylinders. It is easy to check that $\G$ is closed under binary intersections and is, therefore, a $\pi$-system. Also, every cylinder can be decomposed into a finite union of thin cylinders. Thus, the algebra generated by $\G$ is $\CC_0$, hence the $\sigma$-algebra generated by $\G$ is $\CC$. So, by Theorem 3.3 of \cite{Billingsley}, any two probability measures on $\CC$ that agree on $\G$ also agree on $\CC$. Let $x\in\cup_{t=1}^\infty\X^t$. Then $P_2(\C(\{x\}))=p(x)=P_1(\C(\{x\}))$. So, $P_1$ and $P_2$ agree on $\G$, and so they are equal. Thus, $\Phi_{23}$ is the inverse of $\Phi_{32}$.
\end{proof}

Note that the mapping $\Phi_{23}$ is standard. It follows the steps in \cite{Billingsley} almost exactly until the definition of $P_0$ where \cite{Billingsley} instead defines the product measure, which is the independent version of $P_0$. But, in the dependent setting, our $P_0$ is also standard, see e.g. Eq.~(8) of \cite{du2023measure}, as it is the only possible definition that satisfies finite additivity since $\C(A)=\cup_{x\in A}\C(\{x\})$. The final step of extending $P_0$ to $P$ also follows as in \cite{Billingsley}. But the subtle point that we have made in Proposition~\ref{thm:secondthirdkind} is that for every $p\in\class_2$ there is a \textit{unique} $P\in\class_3$ such that $P(\C(\{x\}))=p(x)$ for all $x\in\cup_{t=1}^\infty\X^t$. While the Extension Theorem guarantees that the extension from $\CC_0$ to $\CC$ is unique, which implies that $\Phi_{23}$ is well-defined, it does not prove that the extension from $\G$ to $\CC$ is unique. So, while previous works showed that $\Phi_{23}$ is well-defined, they did not seem to notice that it is a bijection.

Propositions~\ref{thm:firstsecondkind} and~\ref{thm:secondthirdkind} have important philosophical implications. Specifically, assuming that the next word in a sequence of words is always governed by a probability distribution is equivalent to assuming that infinite sequences of words are governed by a distribution. Moreover, the underlying functions in the two cases are in one-to-one correspondence with one another. But, it is conceptually easier to believe that there is an underlying probabilistic language space of the first kind because we naturally think of a sequence of words along with its extra-mathematical context, such as body language and flowers along the road. On the other hand, an underlying probabilistic language space of the third kind implies that language is a stochastic process. To see this, note that every probability measure is the pushforward measure of some random variable: just let $X$ be the identity mapping from $(\Omega,\F)$ to itself, then $PX^{-1}=P$. In particular, for every $P\in\class_3$, there is a probability space $(\Omega,\F,\mu)$ and a $(\X^\N,\CC)$-valued random variable $X$ on $(\Omega,\F,\mu)$ such that $P=\mu X^{-1}$. Each $X_t$ is a $(\X,2^\X)$-valued random variable, so $X$ is a stochastic process. Furthermore, its finite-dimensional distributions $\mu (X_1,\ldots,X_t)^{-1}$ are given by the probability mass function $\Phi_{32}(P)$ for each $t\in\N$.

\cmt{
For the rest of the paper, we will implicitly assume an underlying stochastic sequence while only using the lower-case $p$ for fitting a model to it. Thus, we can think of $p(x)$ as the probability that the stochastic sequence up to index $|x|$ is $x$, i.e. $p(x)=P(\C(\{x\}))$. And we can think of $p(y\mid x)$ as the probability that the value of the stochastic sequence at index $|x|+1$ is $y$ conditioned on the stochastic sequence up to index $|x|$ being $x$, i.e. $p(y\mid x)=P(\C(\{(x,y)\}))/P(\C(\{x\}))$. Here, we are making the usual abuse of notation where $p$ is used to denote both $p$ and $\Phi_{12}(p)$. We do this because the $|$ makes it clear whether the function is in $\class_1$ or $\class_2$. To summarize, if we start with or define $p\in\class_1$, then we use $p$ to denote $\Phi_{12}(p)$ as well, and we can interpret the meaning of both in terms of $\class_3$.
}

\section{Next-token Prediction Capacity}
\label{sec:nexttoken}

From here on out, we will consider the finite set in our probabilistic language space to be $[\omega]=\{1,\ldots,\omega\}$ for some $\omega\in\N$. We call $[\omega]$ a vocabulary (it can also be called an alphabet) and we call its elements tokens. The tokens may represent letters, words, or something in between, as in Byte Pair Encoding \cite{gage1994new}. We assume that there is an underlying probabilistic language space of the first kind, $([\omega],A,p)$, but that we are only given the values of $p$ on $[\omega]\times\{\alpha^1,\ldots,\alpha^n\}\subset [\omega]\times A$ for some $n\in\N$.
We call each $\alpha^i$ a context and we call $[p(\gamma\mid\alpha)]_{\gamma\in[\omega]}\in\Delta^{\omega-1}$ its corresponding next-token distribution. The task of next-token prediction is for a model of the form $\cup_{t=0}^\infty [\omega]^t\to\ri\Delta^{\omega-1}$ to correctly predict next-token distributions given contexts (this implies that if any of the next-token distributions come from the relative boundary of $\Delta^{\omega-1}$, then the model can still approximate them up to arbitrary precision). In particular, the largest $n$ such that it can do so exactly is what we call the next-token prediction capacity. We will formally define this, then prove an upper bound.

\begin{definition}
\label{def:ntpc}
Let $k,\omega\in\N$. For all $\theta\in\R^k$, let
\begin{align*}
    h_\theta:\cup_{t=0}^\infty[\omega]^t\to\ri\Delta^{\omega-1}.
\end{align*}
Then the \textbf{next-token prediction capacity} of the hypothesis class $\{h_\theta\mid\theta\in\R^k\}$ is the largest $n\in\N$ such that for every distinct list $\alpha^1,\ldots,\alpha^n\in\cup_{t=0}^\infty[\omega]^t$ and every list $y_1,\ldots,y_n\in\ri\Delta^{\omega-1}$ there exists $\theta\in\R^k$ such that $y_i=h_\theta(\alpha^i)~\forall i\in[n]$.
\end{definition}

\cmt{
\begin{example}
\label{ex1}
Let 1 represent `I,' 2 represent `love,' 3 represent `hate,' 4 represent `cats,' and 5 represent `dogs.' Let $\alpha^1=(~)$, $\alpha^2=(1)$, $\alpha^3=(1,2)$, and $\alpha^4=(1,3)$. Let $y_1=(1,0,0,0,0)$, $y_2=(0,2/3,1/3,0,0)$, $y_3=(0,0,0,1/2,1/2)$, and $y_4=(0,0,0,1,0)$. This fully defines the first three terms of a stochastic sequence. The first term is `I' with probability 1. The second term is `love' with probability 2/3, `hate' otherwise. If the first two terms are `I' and `love,' then the third term is `cats' with probability 1/2, `dogs' otherwise. If the first two terms are `I' and `hate,' then the third term is `cats' with probability 1.
\end{example}
}

\begin{lemma}
\label{lma:sard}
Let $k,\omega\in\N$. For all $\theta\in\R^k$, let
\begin{align*}
    h_{\theta}:\seqso\to \ri\Delta^{\omega-1}.
\end{align*}
Assume that $\theta\mapsto h_{\theta}(\alpha)$ is continuously differentiable for all $\alpha\in\seqso$.
Then the next-token prediction capacity of $\{h_{\theta}\mid\theta\in\R^k\}$ is less than or equal to $k/(\omega-1)$.
\end{lemma}
\begin{proof}
Let $\alpha^1,\ldots,\alpha^n\in\seqs$ be distinct. Define $F:\R^k\to(\ri\Delta^{\omega-1})^n:\theta\mapsto [h_{\theta}(\alpha^i)]_{i\in[n]}$. $F$ is continuously differentiable, so we can apply Sard's theorem \cmt{(Lemma~\ref{lma:sard1941})}. Note that $\ri\Delta^{\omega-1}$ is an $(\omega-1)$-dimensional manifold. Thus, if $k<n(\omega-1)$, then the image of $F$ has measure zero \cmt{in $(\ri\Delta^{\omega-1})^n$}, hence $F$ is not surjective, proving the theorem.
\end{proof}

In practice, we are generally given a corpus of $N$ document sequences $\beta^1,\ldots,\beta^{\ndoc}\in\seqs$. We can think of these as finite-time samples of the underlying stochastic process. Note that each $\beta^j$ has $|\beta^j|$ contexts $(),(\beta_1^j),(\beta_1^j,\beta_2^j),\ldots$. But, among the corpus, there may be repeated contexts, so we enumerate $\{(\beta_1^j,\ldots,\beta_t^j)\mid j\in[N],t< |\beta^j|\}$ as $\{\alpha^1,\ldots,\alpha^n\}$. We call $\alpha^1,\ldots,\alpha^n$ the unique contexts of the corpus. For all $\alpha\in\seqso$, define $c(\alpha)=|\{j\in[N]\mid \alpha=(\beta_1^j,\ldots,\beta_{|\alpha|}^j)\}|$. This is the number of times that the context $\alpha$ appears in the corpus. In particular, $c(~)=N$. Now we are ready to define the underlying $p\in\class_1$ (where $\class_1$ is as defined in Proposition~\ref{thm:firstsecondkind}) which we denote by $\hat{p}$ to indicate that it comes from a corpus: for all $\gamma\in[\omega]$, for all $\forall \alpha\in\seqso\text{ s.t. }c(\alpha)\neq 0$,
\begin{align}
\label{eq:empirical}
    &\hat{p}(\gamma\mid\alpha) = c(\alpha,\gamma)/c(\alpha).
\end{align}
By the usual abuse of notation, let $\hat{p}$ denote $\Phi_{12}(\hat{p})$ (where $\Phi_{12}$ is as defined in Proposition~\ref{thm:firstsecondkind}). Then $\hat{p}$ satisfies $\hat{p}(\alpha)=c(\alpha)/N~\forall \alpha\in\seqs$.

Given another $q\in\class_1$, we can define the loss function
\begin{align}
\label{eq:crossentropy}
    \ls(\beta^1,\ldots,\beta^N,q) &= -\sum_{j=1}^N\sum_{t=1}^{|\beta^j|}\log q(\beta_t^j\mid\beta_1^j,\ldots,\beta_{t-1}^j).
\end{align}
To see how it relates to the usual notion of the cross-entropy of a probability mass function $q:\X\to[0,1]$ relative to another probability mass function $p:\X\to[0,1]$, $H(p,q)\coloneqq-\sum_{x\in\X}p(x)\log q(x)$, let $T$ denote $\max_{j\in[N]}|\beta^j|$ and observe
\begin{align*}
    \ls(\beta^1,&\ldots,\beta^N,q)\\
    &= -\sum_{t=1}^T \sum_{\alpha\in[\omega]^{t-1}}\sum_{\gamma=1}^\omega c(\alpha,\gamma)\log q(\gamma\mid \alpha)\\
    &= -N\sum_{t=1}^T \sum_{\alpha\in[\omega]^{t-1}}\sum_{\gamma=1}^\omega \hat{p}(\alpha,\gamma)\log q(\gamma\mid \alpha)\\
    &= -\sum_{t=1}^T \sum_{\alpha\in[\omega]^{t-1}}c(\alpha)\sum_{\gamma=1}^\omega \hat{p}(\gamma\mid\alpha)\log q(\gamma\mid \alpha)\\
    &= \sum_{t=1}^T \sum_{\alpha\in[\omega]^{t-1}}c(\alpha)H(\hat{p}(\cdot\mid\alpha), q(\cdot\mid \alpha)).
\end{align*}
The second equality shows that $\ls$ is the sum of the $t$-gram entropies, as defined in Shannon \cite{shannon1948mathematical}, up to $t=T-1$. The last equality shows that $\ls$ is a positively weighted sum of entropies.
Thus, by Gibb's inequality,
\begin{align}
\label{eq:gibbs}
    \ls(\beta^1,\ldots,\beta^{\ndoc},\hat{p})\le \ls(\beta^1,\ldots,\beta^{\ndoc},q)
\end{align}
with equality if and only if $\hat{p}=q$.
This makes the connection between optimization and interpolation sufficiently clear to make the following definition, from which follows a lower bound.

\begin{definition}
\label{def:entpc}
Let $k,\omega\in\N$. For all $\theta\in\R^k$, let
\begin{align*}
    h_\theta:\seqso\to\ri\Delta^{\omega-1}.
\end{align*}
Then the \textbf{empirical next-token prediction capacity} of the hypothesis class $\{h_\theta\mid\theta\in\R^k\}$ is the largest $n\in\N$ such that for every $N\in\N$ and every list $\beta^1,\ldots,\beta^N\in\cup_{t=1}^\infty [\omega]^t$ such that $|\{(\beta_1^j,\ldots,\beta_t^j)\mid j\in[N],t<|\beta^j|\}|\le n$ the following holds:
\begin{align*}
    \inf_{\theta\in\R^k}\ls(\beta^1,\ldots,\beta^N,h_\theta)=\ls(\beta^1,\ldots,\beta^N,\hat{p})
\end{align*}
where $\hat{p}$ and $\ls$ are as defined in Equations~\eqref{eq:empirical} and ~\eqref{eq:crossentropy} respectively.
\end{definition}

\cmt{
\begin{example}
\label{ex2}
Consider a text corpus made up of three documents, `I love cats,' `I love dogs,' and `I hate cats.' If we ignore spaces and tokenize as in Example~\ref{ex1}, then we get $\beta^1=(1,2,4)$, $\beta^2=(1,2,5)$, and $\beta^3=(1,3,4)$. Then $\hat{p}(\cdot\mid \alpha^i)=y_i$ for each $i\in[4]$, where the $\alpha^i$ and $y_i$ are as defined in Example~\ref{ex1}. In this case, $\ls(\beta^1,\beta^2,\beta^3,\hat{p})=3\log(3)$.
\end{example}
}

\begin{lemma}
\label{lma:entropy}
Let $k,\omega\in\N$. For all $\theta\in\R^k$, let
\begin{align*}
    h_\theta:\seqso\to\ri\Delta^{\omega-1}.
\end{align*}
Then the next-token prediction capacity of $\{h_\theta\mid \theta\in\R^k\}$ is less than or equal to the empirical next-token prediction capacity.
\end{lemma}

Thus, any lower bound on the next-token prediction capacity of a model is a lower bound on the empirical next-token prediction capacity of the model as well. On the other hand, we do not yet have a general upper bound on the empirical next-token prediction capacity, so we prove one in the following theorem. As $N$ increases, we show that $[[\hat{p}(\gamma\mid \alpha^i)]_{\gamma\in[\omega]}]_{i\in[n]}$ can approximate any $Y\in(\Delta^{\omega-1})^n$. But, this is not enough to apply Lemma~\ref{lma:sard} since the image of $F$ in the proof, despite being measure zero, may still be dense. To guarantee that it is not dense, we have to further assume that $\theta\mapsto h_{\theta}(\alpha)$ is $(\ran,\exp)$-definable for all $\alpha\in\seqso$.
\cmt{As explained in Section~\ref{sec:prelims}, 
$\phi$, tanh, arctan, erf, and GELU are $(\ran,\exp)$-definable. Furthermore, if $h_\theta$ is the transformer model defined in the next section, and if the activation function is a $(\ran,\exp)$-definable function, such as tanh, arctan, or GELU, then $h_\theta$ is $(\ran,\exp)$-definable.
}

\begin{theorem}
\label{thm:sontag}
Let $k,\omega\in\N$. For all $\theta\in\R^k$, let
\begin{align*}
    h_{\theta}:\seqso\to \ri\Delta^{\omega-1}.
\end{align*}
Assume that $\theta\mapsto h_{\theta}(\alpha)$ is continuously differentiable and $(\ran,\exp)$-definable for all $\alpha\in\seqso$.
Then the empirical next-token prediction capacity of $\{h_\theta\mid\theta\in\R^k\}$ is less than or equal to $(2+1/(\omega-1))k/(\omega-1)+2$.
\end{theorem}
\begin{proof}
Set $n'=\lfloor k/(\omega-1)+1\rfloor$ and $\tdoc=\lceil \log(n')/\log(\omega)\rceil$. Define $\hat{p}$ and $\ls$ as in Eqs.~\eqref{eq:empirical} and~\eqref{eq:crossentropy} respectively.
We will show that there exists $N\ge n'$ and $\beta^1,\ldots,\beta^N\in[\omega]^{N+1}$ such that the number of unique contexts is less than or equal to $(2+1/(\omega-1))k/(\omega-1)+2$ and
\begin{align*}
\inf_{\theta\in\R^k}\ls(\beta^1,\ldots,\beta^{\ndoc},h_\theta)> \ls(\beta^1,\ldots,\beta^{\ndoc},\hat{p}).
\end{align*}
First, the size of $[\omega]^T$ is $\omega^T\ge n'\ge \omega^{T-1}$, so there exists a distinct list $\alpha^1,\ldots,\alpha^{n'}\in[\omega]^T$. Choose such an $\alpha^1,\ldots,\alpha^{n'}$. Let $\beta^1,\ldots,\beta^{n'}\in[\omega]^{T+1}$ such that $\beta^i[1:\tdoc]=\alpha^i~\forall i\in[n']$. Then we can bound the total number of unique contexts
\begin{align*}
    n &\coloneqq n' + \sum_{t=0}^{T - 1}\omega^t\\
    &=n' + \frac{\omega^{T}-1}{\omega-1}\\
    &\le n' + \frac{\omega n'-1}{\omega-1}\\
    &= \frac{2\omega-1}{\omega-1}n'-\frac{1}{\omega-1}\\
    &\le \frac{2\omega-1}{\omega-1}\left(\frac{k}{\omega-1}+1\right)-\frac{1}{\omega-1}.
\end{align*}
Define $F:\R^k\to(\ri\Delta^{\omega-1})^{n'}: \theta\mapsto [h_{\theta}(\alpha^i)]_{i\in[n']}$. $F$ is continuously differentiable and $(\ran,\exp)$-definable. Since $k<n'(\omega-1)$, the image of $F$ is measure zero \cmt{in $(\ri\Delta^{\omega-1})^{n'}$} by Sard's theorem \cmt{(Lemma~\ref{lma:sard1941})}, hence has empty interior. So, by Corollary B.4 of Sontag \cite{sontag1996critical}, it is nowhere dense, hence not dense. Let $Y\in(\Delta^{\omega-1})^{n'}\backslash \overline{F(\R^k)}$. Note that we can include more sequences $\beta^{n'+1},\ldots,\beta^N$ without adding new contexts. Furthermore, we have control over the last token in every sequence, so by increasing $N$, we can make $[[\hat{p}(\gamma\mid \alpha^i)]_{\gamma\in[\omega]}]_{i\in [n']}$ approximate $Y$ to arbitrary precision. Thus, since $Y\notin \overline{F(\R^k)}$, there exists $N\ge n'$ and $\beta^1,\ldots,\beta^N\in[\omega]^{T+1}$ for which $[[\hat{p}(\gamma\mid \alpha^i)]_{\gamma\in[\omega]}]_{i\in [n']}\notin \overline{F(\R^k)}$. So, the conclusion follows from Eq.~\eqref{eq:gibbs}.
\end{proof}

In this section, we have defined next-token prediction capacity and empirical next-token prediction capacity and upper bounded them for general models. The rest of the paper will focus on lower bounding the next-token prediction capacity of the one-layer multi-head decoder-only transformer model, which we define in the next section. Before we move on though, we end the section with remarks on inexact interpolation, finite precision, \cmt{and ``VC dimension.''}

\begin{remark}
Both Lemma~\ref{lma:sard} and Theorem~\ref{thm:sontag} apply even for inexact interpolation. While Lemma~\ref{lma:sard} only uses that $F$, as defined in the proof, is not surjective, the proof actually shows that the image of $F$ has measure zero. If $\theta\mapsto h_\theta(\alpha)$ is additionally assumed to be $(\ran,\exp)$-definable for all $\alpha\in\seqso$, as it is in Theorem~\ref{thm:sontag}, then the image of $F$ is nowhere dense. Similarly, the image of $F$ defined in the proof of Theorem~\ref{thm:sontag} is shown to be nowhere dense. So, for each choice of $k$ and $\omega$, there exists $\epsilon>0$ such that the ``$\epsilon$-approximate'' next-token prediction capacity and ``$\epsilon$-approximate'' empirical next-token prediction capacity are both upper bounded by $O(k/\omega)$.
\end{remark}

\begin{remark}
Any next-token prediction capacity lower bound for a hypothesis class will apply even if the model is constrained to finite precision. To see this, suppose we have a hypothesis class $\{h_\theta\mid\theta\in\R^k\}$ such that $\theta\mapsto h_\theta(\alpha)$ is continuously differentiable for all $\alpha\in\seqso$. If we assume a maximum context length $T$ and restrict $\theta$ to a compact set, then the $\theta\mapsto h_\theta(\alpha)$ are uniformly Lipschitz continuous over $\alpha\in\seqso$, so there exists $\delta>0$ such that any next-token prediction capacity lower bound also holds as an ``$\epsilon$-approximate'' next-token prediction capacity lower bound for the hypothesis class $\{h_\theta\mid\theta\in\R^k\}$ restricted to $\delta$-precision.
\end{remark}

\cmt{
\begin{remark}
It is interesting that the proof of Lemma~\ref{lma:sard} works for every distinct list $\alpha^1,\ldots,\alpha^n$ while in the proof of Lemma~\ref{thm:sontag} we have to construct a particular corpus $\beta^1,\ldots,\beta^N$. In fact, we will show that there is another $\beta^1,\ldots,\beta^N$ with $n$ unique contexts such that $\inf_{\theta\in\R^k}\ls(\beta^1,\ldots,\beta^N,h_\theta)=\ls(\beta^1,\ldots,\beta^N,\hat{p})$. Note that this relates to the concept of Vapnik-Chervonenkis (VC) dimension~\cite{shai2}. If we define next-token prediction VC dimension as the largest $n$ such that there exists a distinct list $\alpha^1,\ldots,\alpha^n$ such that for every list $y_1,\ldots,y_n$ there exists $\theta$ such that $y_i=h_\theta(\alpha^i)~\forall i\in[n]$ (a simple modification to Definition~\ref{def:ntpc}), then the same upper and lower bounds will hold as for next-token prediction capacity. On the other hand, if we define empirical next-token prediction VC dimension as the largest $n$ such that there exists $N$ and $\beta^1,\ldots,\beta^N$ with $n$ unique contexts such that $\inf_{\theta\in\R^k}\ls(\beta^1,\ldots,\beta^N,h_\theta)=\ls(\beta^1,\ldots,\beta^N,\hat{p})$ (a simple modification to Definition~\ref{def:entpc}), then it will be infinite.
Consider the model
$h:\cup_{t=0}^\infty[\omega]^t\to \ri\Delta^{\omega-1}:\alpha\mapsto \one/\omega$. For any $T\in\N$, we can choose our text corpus to be $[\omega]^T$. Then $n=(\omega^{T+1}-1)/(\omega-1)$ and $\ls(\beta^1,\ldots,\beta^N,h)=\ls(\beta^1,\ldots,\beta^N,\hat{p})$. Thus, if $h$ is in the hypothesis class, then the empirical next-token prediction VC dimension of the hypothesis class is infinity. Indeed, it is easy to see how to construct $h$ with the transformer model in the following section, proving that the transformer model has infinite empirical next-token prediction VC dimension.
\end{remark}
}

\section{Transformer Model}
\label{sec:transformer}

The one-layer multi-head decoder-only transformer model is parameterized by the following matrices and vectors:
\begin{itemize}
    \item Let $Z\in\R^{d\times \omega}$. This is the embedding matrix. Its $\gamma$th column corresponds to the embedding of token $\gamma$ in $\R^{d}$. We call $d$ the embedding dimension.
    \item Let $U\in\R^{d\times \N}$. This is the position matrix. While we assume it has infinite columns, it can actually start with finite columns and be extended to arbitrarily many columns by appending new columns whenever a longer input is encountered.
    \item Let $W_{1,r},W_{2,r}\in\R^{d\times d_r},W_{3,r}\in\R^{d\times d_0}~\forall r\in[m_0],W_0\in\R^{m_0d_0\times d}$. These are the self-attention matrices. We call $m_0$ the number of heads.
    \item Let $W\in\R^{d\times m},b\in\R^m,V\in\R^{m\times \omega}$. These are the hidden matrix, bias vector, and output matrix respectively. We call $m$ the number of neurons.
\end{itemize}
Additionally, let $\psi:\R\to\R$ be real analytic at a point and not a polynomial there. This is the activation function. For example, $\psi$ could be a sigmoid, such as the logistic function, tanh, or arctan, or it could be a smoothed rectified linear unit, such as GELU.

Let $\theta$ denote the concatenation of all the model parameters. Define
\begin{align}
\begin{split}
\label{eq:model}
    &h_{1,\theta}:\seqs\to \cup_{t=1}^\infty \R^{d\times t}:\alpha\mapsto Z[:,\alpha]+U[:,1:|\alpha|],\\
    &h_{2,\theta}:\cup_{t=1}^\infty \R^{d\times t}\to\R^d:X\mapsto\\
    &\hspace{.5cm} W_0^\top\tvec\left(\left[W_{3,r}^\top X\phi\left(X^\top W_{1,r}W_{2,r}^\top X[:,-1]\right)\right]_{r\in [m_0]}\right),\\
    &h_{3,\theta}:\R^d\to\R^\omega:x\mapsto V^\top\phi\left(W^\top x+b\right),\\
    &\text{and }h_\theta\coloneqq \phi\circ h_{3,\theta}\circ h_{2,\theta}\circ h_{1,\theta}
    \end{split}
\end{align}
where $-1$ denotes the last index and $\phi$ is the softmax function. We think of the value $h_{\theta}(~)$ as an additional $\omega-1$ model parameters. The first three functions are the embedding, self-attention, and FNN sub-layers respectively of the one-layer transformer $h_\theta$.
Note that we do not include the skip-connection $+X[:,-1]$ in the self-attention sub-layer. This was omitted for simplicity, but its inclusion does not affect Theorems~\ref{thm:injective} and~\ref{thm:full}, which follow by almost the same proof.

\section{Transformer Next-token Prediction Capacity}
\label{sec:memory}

In order to lower bound the next-token prediction capacity of the transformer model, we first reduce to the special case of $d=d_1=d_0=m_0=1$ and $W_{1,1}=W_{2,1}=W_{3,1}=W_0=1$. Thus, \cmt{$h_\theta$} becomes the following composition of mappings:
\begin{align*}
    \underset{\seqs}{\rule[-.3cm]{.1pt}{1cm}}\overset{f(z,u,\cdot)}{\looongrightarrow}\underset{\R}{\bm{\cdot}}\overset{\psi\left(\cdot w+b\right)}{\looongrightarrow}\underset{\R^{m}}{\rule[-1.3cm]{.1pt}{3cm}} \overset{V^\top\cdot}{\looongrightarrow}\underset{\R^{\omega}}{\rule[-.9cm]{.1pt}{2cm}}\overset{\phi}{\looongrightarrow}\underset{\ri\Delta^{\omega-1}}{\rule[-.9cm]{.1pt}{2cm}}
\end{align*}
where $f(z,u,\alpha)=\langle x,\phi(xx_{|\alpha|})\rangle$ with $x=z[\alpha]+u[1:|\alpha|]$.
\cmt{
To relate this to the previous section,
$Z=z^\top$, $U=u^\top$, $W=w^\top$, and $h_{2,\theta}\circ h_{1,\theta}=f(z,u,\cdot)$.
}
Assume $\psi:\R\to\R$ is real analytic at some $\eta\in\R$ and not a polynomial there. Assume $m\ge n$. We will show that this simplified model can interpolate arbitrary data sets of size $n$. Then we will use this to show that the full model can interpolate arbitrary data sets of size $n$.

To show the former, let $\alpha^1,\ldots,\alpha^n\in\seqs$ and $Y'\in(\ri\Delta^{\omega-1})^n$. Then there exists $Y\in\R^{\omega\times n}$ such that $\phi(Y)=Y'$ with $\phi$ applied column-wise. We will construct an interpolating model. First, set $b=\eta\one_m$. Second, define $T=\max_{i\in[n]}|\alpha^i|$ and sample $(z,u)\in\R^\omega\times\R^T$ from any continuous probability distribution. Let $\hat{x}_1,\ldots,\hat{x}_n$ be the corresponding outputs of the first mapping. Then the matrix of outputs of the composition of mappings is $\phi(V^\top\psi(w[\hat{x}_i]_{i\in[n]}^\top+b\one_n^\top))$. In Theorem~\ref{thm:injective}, we prove that the $\hat{x}_i$ are nonzero and distinct. In Theorem~\ref{thm:realrank}, we prove that this implies $\psi(w[\hat{x}_i]_{i\in[n]}^\top+b\one_n^\top)$ has full rank for generic $w$ in an open set around zero. So, as our third step in the construction, sample $w\in\R^m$ from any continuous probability distribution that has support in the open set. Then, $\psi(w[\hat{x}_i]_{i\in[n]}^\top+b\one_n^\top)$ has rank $n$ so, as our fourth and final step in the construction, solve for $V$ in $Y=V^\top\psi(w[\hat{x}_i]_{i\in[n]}^\top+b\one_n^\top)$ as a system of linear equations.

Thus, the simplified model is able to interpolate arbitrary data sets as long as $m\ge n$. To see how to reduce the full model with $d,m_0,d_0\in\N$ and $d_r\in\N~\forall r\in[m_0]$ to the simplified model, set $Z=\one_d z^\top/\sqrt{d}, U=\one_d u^\top/\sqrt{d},W_{1,r}=W_{2,r} = \one_{d}\one_{d_r}^\top/\sqrt{dd_r}~\forall r\in[m_0], W_{3,r}=\one_d\one_{d_0}^\top/\sqrt{dd_0}~\forall r\in[m_0], W_0=\one_{m_0d_0}\one_d^\top/(m_0\sqrt{dd_0}),W=\one_d w^\top/\sqrt{d}$. In other words, token $\gamma\in[\omega]$ is embedded as $z_{\gamma}\one_d/\sqrt{d}$ and the $t$th position vector is $u_t\one_d/\sqrt{d}$. So, the inputs to the self-attention sub-layer are sequences of vectors coming from a one-dimensional subspace, and so it makes sense for the self-attention matrices to be constant matrices. Finally, the $k$th hidden neuron of the FNN sub-layer has weight vector $w_k\one_d/\sqrt{d}$. By reducing to the simplified model in this way, and applying Theorems~\ref{thm:injective} and~\ref{thm:realrank}, we are able to lower bound the next-token prediction capacity of the one-layer multi-head decoder-only transformer model by $m$ in Theorem~\ref{thm:full}.

To summarize the proof sketch, each sub-layer has a particular role in the proof and these roles are:
\begin{enumerate}
    \item the embedding and self-attention sub-layers map distinct sequences to distinct vectors,
    \item the hidden FNN layer maps these vectors to a larger space where they are linearly independent,
    \item the outer FNN layer is solved as a system of linear equations.
\end{enumerate}

\subsection{Injectivity of Self-attention}

To prove that the embedding and self-attention layers can take on the first role in the proof sketch, we have the following theorem.

\begin{theorem}
\label{thm:injective}
Let $\omega,T\in\N$. Let $\phi$ be the softmax function. Define
\begin{align*}
    f:~&\R^{\omega}\times \R^{T }\times\cup_{t=1}^{T }[\omega]^t\to\R\\
    &(z,u,\alpha)\mapsto \langle x,\phi(xx_{|\alpha|})\rangle\\
    &\text{where }x=z[\alpha]+u[1:|\alpha|].
\end{align*}
Then $f(z,u,\cdot)$ maps injectively into $\R\backslash\{0\}$ for generic $(z,u)$.
\end{theorem}
\begin{proof}
First, let $A$ denote $\cup_{t=1}^{T }[\omega]^t$. Then, note that, for all $z\in\R^{\omega},u\in\R^{T },\tau\in[T ],\alpha\in[\omega]^{\tau}$,
\begin{align*}
    f(z,u,\alpha) = \frac{\sum_{t=1}^{\tau}(z_{\alpha_t}+u_t)\exp((z_{\alpha_t}+u_t)(z_{\alpha_{\tau}}+u_{\tau}))}{\sum_{t=1}^{\tau}\exp((z_{\alpha_t}+u_t)(z_{\alpha_{\tau}}+u_{\tau}))}.
\end{align*}
Clearly, $f(\cdot,\cdot,\alpha)$ is real analytic for all $\alpha \in A$, hence
\begin{align*}
    g\coloneqq \left(\prod_{\alpha\in A}f(\cdot,\cdot,\alpha)\right)\left(\prod_{\alpha,\beta\in A:\alpha\neq \beta}(f(\cdot,\cdot,\alpha)-f(\cdot,\cdot,\beta))\right)
\end{align*}
is real analytic. Note that the set of $(z,u)$ for which $f(z,u,\cdot)$ does not map injectively into $\R\backslash\{0\}$ is precisely the zero set of $g$. Thus, the result follows if and only if $g$ is nontrivial if and only if (1) each $f(\cdot,\cdot,\alpha)$ is nontrivial and (2) each $f(\cdot,\cdot,\alpha)-f(\cdot,\cdot,\beta)$ is nontrivial. To prove (1), observe that $f(\one_{\omega},0,\alpha)=1~\forall \alpha\in A$. Now, we will go about proving (2). First, let $\alpha,\beta\in A$. We want to show that $\alpha\neq \beta$ implies $f(\cdot,\cdot,\alpha)-f(\cdot,\cdot,\beta)$ is nontrivial. We will prove the contrapositive. Suppose $f(\cdot,\cdot,\alpha)-f(\cdot,\cdot,\beta)\equiv 0$. First, we will show that $\alpha$ and $\beta$ have the same length, then we will show that each entry of $\alpha-\beta$ is zero.

To show that $\alpha$ and $\beta$ have the same length, set $z=0$ and $u=(1,\ldots,T )$, and define the polynomials
\begin{align*}
    &p_{\gamma} = \sum_{t=1}^{|\gamma|}tx^{t|\gamma|}\text{ and }q_{\gamma} = \sum_{t=1}^{|\gamma|}x^{t|\gamma|}~\forall \gamma\in\{\alpha,\beta\}.
\end{align*}
Note that the monomial $x^{|\alpha|^2+|\beta|^2}$ has coefficient $|\alpha|-|\beta|$ in $p_{\alpha}q_{\beta}-p_{\beta}q_{\alpha}$. Also, for all $\gamma\in\{\alpha,\beta\}$, $f(z,u,\gamma)=p_{\gamma}(\exp(1))/q_{\gamma}(\exp(1))$. So, since $f(\cdot,\cdot,\alpha)-f(\cdot,\cdot,\beta)\equiv 0$, $(p_{\alpha}q_{\beta}-p_{\beta}q_{\alpha})(\exp(1))=0$. But, $\exp(1)$ is transcendental, so $p_{\alpha}q_{\beta}-p_{\beta}q_{\alpha}\equiv 0$. Thus, $|\alpha|=|\beta|$.

Let $\tau=|\alpha|=|\beta|$. Let $s\in[\tau]$. To show that $\alpha_s=\beta_s$, set $z=(1,\ldots,\omega)$ and $u=\omega e_s$, and define the polynomials
\begin{align*}
    &p_{\gamma} = \sum_{t=1}^{\tau}(\gamma_t+\omega\delta_{t,s})x^{(\gamma_t+\omega\delta_{t,s})(\gamma_{\tau}+\omega\delta_{\tau,s})}\\
    \text{and }&q_{\gamma} = \sum_{t=1}^{\tau}x^{(\gamma_t+\omega\delta_{t,s})(\gamma_{\tau}+\omega\delta_{\tau,s})}~\forall \gamma\in\{\alpha,\beta\}.
\end{align*}
Note that the monomial $x^{(\alpha_s+\omega)(\alpha_{\tau}+\omega\delta_{\tau,s})+(\beta_s+\omega)(\beta_{\tau}+\omega\delta_{\tau,s})}$ has coefficient $\alpha_s-\beta_s$ in $p_{\alpha}q_{\beta}-p_{\beta}q_{\alpha}$. Also, $f(z,u,\gamma)=p_{\gamma}(\exp(1))/q_{\gamma}(\exp(1))~\forall \gamma\in\{\alpha,\beta\}$. So, again, $p_{\alpha}q_{\beta}-p_{\beta}q_{\alpha}\equiv 0$. Thus, $\alpha_s=\beta_s$. Since this was for arbitrary $s\in[\tau]$, we get that $\alpha=\beta$, proving the theorem.
\end{proof}

After assuming $f(\cdot,\cdot,\alpha)\equiv f(\cdot,\cdot,\beta)$, the main idea of the rest of the proof was that
\begin{align*}
    &\frac{\sum_{t=1}^{\tau}(z_{\alpha_t}+u_t)\exp((z_{\alpha_t}+u_t)(z_{\alpha_{\tau}}+u_{\tau}))}{\sum_{t=1}^{\tau}\exp((z_{\alpha_t}+u_t)(z_{\alpha_{\tau}}+u_{\tau}))}\\
    &\hspace{.5cm}=\frac{\sum_{t=1}^{\tau}(z_{\beta_t}+u_t)\exp((z_{\beta_t}+u_t)(z_{\beta_{\tau}}+u_{\tau}))}{\sum_{t=1}^{\tau}\exp((z_{\beta_t}+u_t)(z_{\beta_{\tau}}+u_{\tau}))}
\end{align*}
for all $(z,u)$, so when we multiply denominators and move everything to one side, we get that $\exp(1)$ is the root of a polynomial. Since $\exp(1)$ is transcendental, the polynomial is the zero polynomial, so each monomial has coefficient equal to zero. Thus, the idea is to choose $(z,u)$ such that the coefficient of a particular monomial being zero is precisely what we are trying to show. Moreover, we can do this for as many $(z,u)$'s as we want. First, we set $z=0$ and $u=(1,\ldots,T)$, targeting the lengths of $\alpha$ and $\beta$. Here, it makes sense that only the positional embedding is needed. Then, we set $z=(1,\ldots,\omega)$ and $u=\omega e_s$, using the positional embedding to target position $s$. These choices of $(z,u)$ turn out to be sufficient to show that $\alpha=\beta$. 

\begin{remark}
It is quite interesting that we needed the transcendence of $\exp(1)$ for the proof and it is natural to wonder what would happen if we replaced $\phi$ with the function $x\mapsto a^x/\sum_{\gamma=1}^{\omega}a^{x_{\gamma}}$ for some $a>1$ which is not transcendental. In this case, since $a^b=\exp(b\log(a))$, we would just need to scale our example $(z,u)$'s by $\sqrt{\log(a)}$ to get the same next-token prediction capacity result. So, the important aspect of self-attention that leads to its injectivity is that $\phi$ has the form $x\mapsto a^x/\sum_{\gamma=1}^{\omega}a^{x_{\gamma}}$, not that the base happens to be $\exp(1)$.
\end{remark}

It is also natural to wonder what is the simplest alternative to self-attention that retains injectivity. We will explore this question later in Section~\ref{sec:tokenavg}, but for now we continue with the next-token prediction capacity of the transformer model.

\subsection{Rank Results}

To prove that the hidden FNN layer can take on the second role in the proof sketch, we have the following two theorems. In the first theorem, we derive the generic rank of $\psi(ab^\top)$ for polynomial $\psi$. In the second theorem, we extend to $\psi$ which are real analytic at zero and not a polynomial there.

\begin{theorem}
\label{thm:polyrank}
Let $K\subset \N\cup\{0\}$ and $c_k\in\R\backslash\{0\}~\forall k\in K$. Define $\psi(x)=\sum_{k\in K}c_kx^k$. Let $b\in \R^n$ have entries which are nonzero and distinct. Then the rank and Kruskal rank of $\psi(ab^\top)$ are $\min\{m,n,|K|\}$ for generic $a\in\R^m$.
\end{theorem}
\begin{proof}
Observe,
\begin{align*}
    \psi\left(ab^\top\right) &= \sum_{k\in K}c_k\left(ab^\top\right)^{(k)} = \sum_{k\in K}c_ka^{(k)}b^{(k)\top}.
\end{align*}
Let $s=\min\{m,n,|K|\}$. Let $I\subset [m]$ and $J\subset [n]$ such that $|I|=|J|=s$. Enumerate $I$ as $\{i_1,\ldots,i_s\}$ and $J$ as $\{j_1,\ldots,j_s\}$. We want to show that the $(I,J)$ minor is not identically zero. By the Cauchy-Binet formula,
\begin{align*}
    p(a) &\coloneqq \det_{I,J}\left(\psi\left(ab^\top\right)\right)\\
    &= \sum_{S\in\binom{K}{s}}\left(\prod_{k\in S}c_k\right)\underset{\coloneqq p_S(a)}{\underbrace{\det_{I,[s]}\left(\left[a^{(k)}\right]_{k\in S}\right)}}\underset{\coloneqq q_S}{\underbrace{\det_{J,[s]}\left(\left[b^{(k)}\right]_{k\in S}\right)}}.
\end{align*}
Note that this would be identically zero if $s$ were greater than $|K|$ (since it would be an empty sum), thus the rank and Kruskal rank of $\psi(ab^\top)$ are less than or equal to $\min\{m,n,|K|\}$ for all $a\in\R^m$. To go about proving the lower bound, first let $S\in\binom{K}{s}$ and enumerate it as $\{k_1,\ldots,k_s\}$. Then, applying the Leibniz determinant formula, we get
\begin{align*}
    p_S(a) &= \sum_{\sigma\in\sym(s)}\sgn(\sigma)\underset{\coloneqq p_{S,\sigma}(a)}{\underbrace{a_{i_1}^{k_{\sigma(1)}}\cdots a_{i_s}^{k_{\sigma(s)}}}}.
\end{align*}

Let $\tau\in\sym(s)$. If $p_{S,\sigma}=p_{S,\tau}$, then $k_{\sigma(t)}=k_{\tau(t)}~\forall t\in[s]$, so $\sigma=\tau$, and so, collecting monomial terms, $p_{S,\sigma}$ has coefficient $\sgn(\sigma)$ in $p_S$. Let $S'\in\binom{K}{s}$. If $p_{S,\sigma}=p_{S',\tau}$, then, similarly, $S=S'$ and $\sigma=\tau$. Thus, if $S'\neq S$, then the monomial $p_{S,\sigma}$ has coefficient zero in $p_{S'}$. So $(p_S)$ is linearly independent. And the $q_S$ are nonzero since they are minors of a Vandermonde matrix with roots which are nonzero and distinct, so $p\not\equiv 0$, completing the proof.
\end{proof}

\begin{theorem}
\label{thm:realrank}
Let $\psi:\R\to\R$ be real analytic at zero and not a polynomial there. Let its radius of convergence at zero be $\rho$. Let $b\in \R^n$ have entries which are nonzero and distinct. Define $M=\{a\in\R^m\mid |a_ib_j|<\rho~\forall (i,j)\}$. Then $M$ is open and the rank and Kruskal rank of $\psi(ab^\top)$ are $\min\{m,n\}$ for generic $a\in M$.
\end{theorem}
\begin{proof}
To see that $M$ is open, notice that it is the pre-image of $(-\rho,\rho)^{m\times n}$ under the continuous map $a\mapsto ab^\top$. Now, let $(c_k)$ be the coefficients of the Taylor expansion at zero of $\psi$. Define $K=\{k\in\N\cup\{0\}\mid c_k\neq 0\}$. Define $s=\min\{m,n\}$ and let $k_1,\ldots,k_s$ be the first $s$ integers in $K$. Define $K_s=\{k_1,\ldots,k_s\}$. Let $I\subset [m]$ and $J\subset [n]$ such that $|I|=|J|=s$. Define $p:M\to\R:a\mapsto \det_{I,J}(\sum_{k\in K_s} c_k(ab^\top)^{(k)})$ and $q:M\to R:a\mapsto \det_{I,J}(\psi(ab^\top))$. Note from the proof of Theorem~\ref{thm:polyrank} that the monomials of $p$ all have degree $k_1+\cdots+k_s$ whereas if $S\in \binom{K}{s}$ has any integers not in $K_s$, then its corresponding monomials will have degree $>k_1+\cdots+k_s$. Thus, $p$ is the truncation of the Taylor expansion at zero of $q$ to precisely the monomials of degree $k_1+\cdots+k_s$. Since, by Theorem~\ref{thm:polyrank}, $p$ is not identically zero, it has at least one nonzero coefficient, and so the Taylor expansion at zero of $q$ has at least one nonzero coefficient. Thus, $q$ is not identically zero, completing the proof.
\end{proof}

\subsection{Interpolation from the FNN}

Combining Theorem~\ref{thm:injective} and Theorem~\ref{thm:realrank} leads to a next-token prediction capacity lower bound for one-layer multi-head decoder-only transformers.

\begin{theorem}
\label{thm:full}
Let $d,m_0,d_0,m,\omega\in\N$ and $d_r\in\N~\forall r\in[m_0]$.
Let $\psi:\R\to\R$ be real analytic at some $\eta\in\R$ and not a polynomial there.
Let $h_{\theta}$ be the one-layer multi-head decoder-only transformer model defined in Equation~\eqref{eq:model}.
Then the next-token prediction capacity of $\{ h_{\theta}\mid\theta\}$ is greater than or equal to $m$.
\end{theorem}
\begin{proof}
Set $n=m$. Let $\alpha^1,\ldots,\alpha^n\in\seqs$ be distinct and define $T=\max_{i\in[n]}|\alpha^i|$.
Invoking Theorem~\ref{thm:injective}, choose $(z,u)$ such that $[f(z,u,\alpha^i)]_{i\in[n]}$ has entries which are nonzero and distinct. Invoking Theorem~\ref{thm:realrank}, choose $w$ such that $\psi(w[f(z,u,\alpha^i)]_{i\in[n]}^\top+\eta\one_m\one_n^\top)$ has rank $n$.
Let $Y'\in(\ri\Delta^{\omega-1})^n$. Then there exists $Y\in\R^{\omega\times n}$ such that $\phi(Y)=Y'$ where $\phi$ is the column-wise softmax function. Furthermore, there exists $V\in\R^{m\times \omega}$ such that $Y=V^\top\psi(w[f(z,u,\alpha^i)]_{i\in[n]}^\top+\eta\one_m\one_n^\top)$. So, by defining $Z=\one_d z^\top/\sqrt{d}, U=\one_d u^\top/\sqrt{d},W_{1,r}=W_{2,r} = \one_{d}\one_{d_r}^\top/\sqrt{dd_r}~\forall r\in[m_0], W_{3,r}=\one_d\one_{d_0}^\top/\sqrt{dd_0}~\forall r\in[m_0], W_0=\one_{m_0d_0}\one_d^\top/(m_0\sqrt{dd_0}),W=\one_d w^\top/\sqrt{d},b=\eta\one_m$, we get that $Y'=[h_\theta(\alpha^i)]_{i\in[n]}$, completing the proof.
\end{proof}

Theorem~\ref{thm:full} shows how the one-layer transformer model can interpolate $m$ input-output pairs through the following mechanism: (1) the embedding and self-attention sub-layers map distinct sequences to distinct vectors, (2) the hidden FNN layer maps these vectors to a larger space where they are linearly independent, (3) the outer FNN layer is solved as a system of linear equations.

Theorem~\ref{thm:full} also suggests an algorithm for finding an interpolating solution. First, note that every minor of $\psi(W^\top[(h_{2,\theta}\circ h_{1,\theta})(\alpha^i)]_{i\in[n]}+b\one_n^\top)$ is real analytic on an open neighborhood of zero. Thus, it works to (1) sample $\theta$ from any continuous probability distribution, (2) choose $\epsilon>0$ sufficiently small so that the entries of $\epsilon W^\top[(h_{2,\theta}\circ h_{1,\theta})(\alpha^i)]_{i\in[n]}+b\one_n^\top$ are each in the interval of convergence of $\psi$, and (3) solve for $V$ in $Y=V^\top\psi(W^\top[(h_{2,\theta}\circ h_{1,\theta})(\alpha^i)]_{i\in[n]}+b\one_n^\top)$.

\cmt{Note that this is not a useful algorithm in the empirical setting because it requires that $\hat{p}$ be explicitly computed. And we are not claiming that we should try to use a transformer in the same way that it is used in our proof. In fact, as we will see in the next section, the embedding and self-attention sub-layers seem to be useful for more than just injectivity, which is, of course, expected. However, the proof does suggest a conceptual framework for guiding the design of network architectures. As long as each part of the proof continues to be satisfied, there is no reason not to swap out sub-layers for new mappings that have not been tried before. For example, in Section~\ref{sec:tokenavg} we will show that the self-attention sub-layer can be swapped with something much simpler without loss of the next-token prediction capacity lower bound. While it may not be a practical substitution, it suggests that there may be more practical substitutions out there.

Finally, it is time to circle back to Section~\ref{sec:nexttoken} and 
}
compare the lower bound of Theorem~\ref{thm:full} to the upper bound of Lemma~\ref{lma:sard}. \cmt{In order to do so,} we have to assume a maximum context length $T$. Also, let $d_r=d_1~\forall r\in[m_0]$ for simplicity. Then, the total number of parameters is $\omega m+m(d+1)+2m_0(d_0+d_1)d+(\omega+\tdoc)d$. Thus, the ratio between the upper and lower bounds is
\begin{align}
\label{eq:const}
    1+\frac{d+2}{\omega-1}+\frac{2m_0(d_0+d_1)d}{(\omega-1)m}+\frac{(\omega+\tdoc)d}{(\omega-1)m},
\end{align}
which is $O(1)$ if and only if the number of parameters in the FNN is at least a constant proportion of the total number of parameters. To see that this is generally the case, consider how, in practice, $\min\{\omega,m\}\ge d\ge \max\{d_0,d_1\}$ and $\omega m\ge \tdoc d$, so the ratio reduces to $O(m_0d^2/(\omega m))$, which is more clearly $O(1)$.

\section{Experiment}
\label{sec:experiment}

We also provide an experiment showing that $\Omega(n)$ parameters seems to be sufficient not only in terms of memory capacity, but also in terms of training.
\cmt{Taking a subset from a text corpus, we train the transformer model of Section~\ref{sec:transformer} on it. We increase the hidden dimension, $m$, until we can get the training error close to the entropy lower bound. By increasing the subset size, we can see how the number of parameters required to reach the entropy lower bound depends on the number of unique contexts.
}
In the line plot of Figure~\ref{fig-1}, \cmt{the dependence} seems to be roughly linear, suggesting that first-order optimization methods can find an interpolating solution almost as soon as one exists.
\cmt{Note that increasing the subset size not only increases the number of unique contexts, but also increases the vocabulary size, and so increases the number of parameters. Thus, $m$ does not have to increase very much to maintain the same training error, as can be seen in the heatmap of Figure~\ref{fig-1}. To see more clearly how $m$ affects the training error, we repeat the experiment but only train the FNN parameters this time. Again, we see a roughly linear dependence in the line plot of Figure~\ref{fig-2}, but we can also see that dependence more clearly in the heatmap. Note that training only the FNN parameters is similar to our proof technique where the embedding and self-attention sub-layers were only used for mapping injectively, not for training. The experiment shows that while this can work in practice, the slope in the line plot is roughly five times larger when only the FNN parameters are trained, so it is preferable to train all the parameters.}

\begin{figure*}
    \centering
    \includegraphics[scale=1]{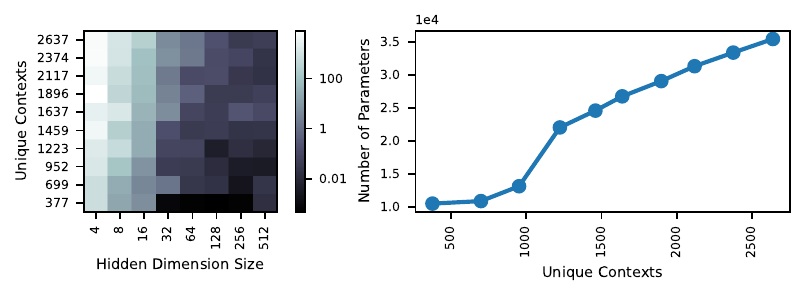}
    \caption{We show that our model requires more parameters to memorize an increasing number of unique contexts. Left: As the hidden dimension, $m$, increases and the number of unique contexts, $n$, decreases, the gap between the training error and the entropy lower bound trends downwards. Right: As the number of unique contexts increases, the minimum number of parameters required for the gap between the training error and the entropy lower-bound to fall below the minimum threshold increases.}
    \label{fig-1}
\end{figure*}

\begin{figure*}
    \centering
    \includegraphics[scale=1]{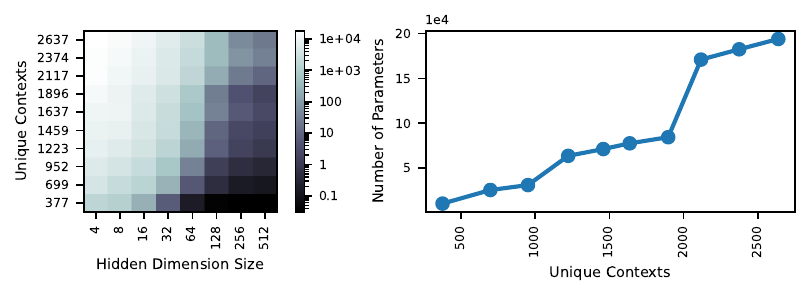}
    \caption{\cmt{Even when only training the FNN layers, we show that our model can memorize an increasing number of unique contexts as the hidden dimension $m$ is increased. Left: As the hidden dimension, $m$, increases and the number of unique contexts, $n$, decreases, the gap between the training error and the entropy lower bound trends downwards. Right: As the number of unique contexts increases, the minimum number of parameters required for the gap between the training error and the entropy lower-bound to fall below the minimum threshold increases (this trains both FNN linear layers, will check if only training the last layer works as well and let you know).}}
    \label{fig-2}
\end{figure*}

All experiments are performed on the TinyStories data set \cite{eldan2023tinystories}, where we take subsets of size 100, 200, 300, 400, 500, 600, 700, 800, 900, 1000 from the data set \cmt{and tokenize using WordPunctTokenizer from the Natural Language Toolkit. For simplicity and reduced computation time, we truncate each sequence to have a length of 10.} These data subsets have 377, 699, 952, 1223, 1459, 1637, 1896, 2117, 2374, 2637 unique contexts, respectively. \cmt{Finally, the corresponding vocabulary sizes are 182, 290, 358, 417, 469, 513, 560, 606, 648, 690, respectively.} Note that in our line plot \cmt{in Figure~\ref{fig-1}}, the minimum threshold varies for each data subset and is computed by taking the entropy lower bound and multiplying it by 0.01 (this helps to account for the fact that larger data sets have larger entropy lower bounds). \cmt{We perform the same calculation for the line plot in Figure~\ref{fig-2} instead multiplying by 0.05.} We train 8 transformer models for each data subset, with $m$ values 4, 8, 16, 32, 64, 128, 256, 512. Finally, we use an embedding size of 16 and the sinusoidal positional embedding from \cite{vaswani2017attention}.

The implementation of our model is a decoder-only transformer and does not have any encoder blocks. It has only one decoder block, which consists of a self-attention layer followed by a linear layer, after which a GELU activation function is applied. Only one head is used for the self-attention. The output of this decoder block is fed into another linear layer, which produces the final output. We include a bias in both linear layers. Unlike many of the more complicated transformer models, we do not use any dropout, layer normalization, or residual connections in our implementation. \cmt{The total number of parameters is $\omega m+17(\omega+m)+1088$ where $\omega$ is the vocabulary size. The total number of FNN parameters is $\omega m+\omega+17m$. The proportion of parameters that are in the FNN is around two thirds for the data points in the line plot of Figure~\ref{fig-1}, and increases from around two thirds to 94\% for the data points in the line plot of Figure~\ref{fig-2}.
}

To compute the training error on each iteration, the cross-entropy loss function is used. For optimization, we use the Adam optimizer \cite{kingma2015adam} in full batch mode, with a stepsize of 0.0001 and no regularization (sometimes referred to as weight decay). Our training error values are computed based on 50,000 iterations. We choose a large number of iterations since our data subsets are small and we want to see how close we can get to the entropy lower bound for each value of $m$. \cmt{Finally, we run our experiments on a computer with a Intel(R) Core(TM) i9-14900X processer and a NVIDIA GeForce RTX 4090 GPU with 24GB of memory.}
We refer you to the following Github link for the code to run our experiments: \texttt{https://github.com/curtfox/decoder-memory-capacity}.

\section{Token-averaging}
\label{sec:tokenavg}

In the proof of Theorem~\ref{thm:full}, all that was required of the embedding and self-attention sub-layers was that the composition of them maps injectively into $\R\backslash\{0\}$. This suggests a simple alternative to self-attention. Let $Z\in\R^{d\times\omega}$, $u\in\R^\N$, $W\in\R^{d\times m}$, $b\times\R^m$, and $V\in\R^{m\times\omega}$. Let $\theta$ denote $(Z,u,W,b,V)$. Let $\psi:\R\to\R$ be real analytic at a point and not a polynomial there. Define
\begin{align*}
    h_\theta:\seqs&\to\ri\Delta^{\omega-1}\\
    \alpha&\mapsto \phi\left(V^\top\psi\left(W^\top Z[:,\alpha]u[1:|\alpha|]+b\right)\right).
\end{align*}
We call this a token-averaged FNN. Essentially, we have extended the FNN to be defined on $\seqs$ through the mapping $\alpha\mapsto Z[:,\alpha]u[1:|\alpha|]$, which we call token-averaging, and which we replaced self-attention with. It turns out this simple model also has optimal next-token prediction capacity, as we now prove.

\begin{theorem}
\label{thm:avginj}
Let $\omega,T \in\N$. Define
\begin{align*}
    f:~&\R^{\omega}\times \R^{T }\times\cup_{t=1}^{T }[\omega]^t\to\R\\
    &(z,u,\alpha)\mapsto \sum_{t=1}^{|\alpha|}u_tz_{\alpha_t}.
\end{align*}
Then $f(z,u,\cdot)$ maps injectively into $\R\backslash\{0\}$ for generic $(z,u)$.
\end{theorem}
\begin{proof}
First, let $A$ denote $\cup_{t=1}^{T }[\omega]^t$ and define
\begin{align*}
    g= \left(\prod_{\alpha\in A}f(\cdot,\cdot,\alpha)\right)\left(\prod_{\alpha,\beta\in A:\alpha\neq \beta}(f(\cdot,\cdot,\alpha)-f(\cdot,\cdot,\beta))\right).
\end{align*}
Note that the set of $(z,u)$ for which $f(z,u,\cdot)$ does not map injectively into $\R\backslash\{0\}$ is precisely the zero set of $g$. Thus, the result follows if and only if $g$ is nontrivial if and only if (1) each $f(\cdot,\cdot,\alpha)$ is nontrivial and (2) each $f(\cdot,\cdot,\alpha)-f(\cdot,\cdot,\beta)$ is nontrivial. Note that $f(\one_{\omega},\one_{T },\alpha)=|\alpha|$. This proves (1). It also proves (2) in the case where $|\alpha|\neq |\beta|$. All that is left is to prove (2) in the case where $|\alpha|=|\beta|$. Let $\alpha,\beta\in A$ such that $\alpha\neq\beta$ and $\tau\coloneqq |\alpha|=|\beta|$. Then there exists $s\in[\tau]$ such that $\alpha_s\neq \beta_s$. Choose such an $s$. Set $z=(1,\ldots,\omega)$ and $u=e_s$. Then $f(z,u,\alpha)=\alpha_s\neq \beta_s=f(z,u,\beta)$, proving (2), and so proving the theorem.
\end{proof}

\begin{theorem}
\label{thm:avgmem}
Let $d,m,\omega\in\N$. For all $\theta\coloneqq (Z,u,W,b,V)\in\R^{d\times \omega}\times \R^{\N}\times \R^{d\times m}\times\R^m\times\R^{m\times\omega}$, define
\begin{align*}
    h_{\theta}:\seqs&\to\R^{\omega}\\
    \alpha&\mapsto \phi\left(V^\top\psi\left(W^\top\sum_{t=1}^{|\alpha|}u_tz_{\alpha_t}+b\right)\right)
\end{align*}
where $\phi$ is the softmax function and $\psi:\R\to\R$ is real analytic at some $\eta\in\R$ and not a polynomial there. Then the next-token prediction capacity of $\{h_{\theta}\mid\theta\}$ is greater than or equal to $m$.
\end{theorem}
\begin{proof}
Set $n=m$. Let $\alpha^1,\ldots,\alpha^n\in\seqs$ be distinct and define $T=\max_{i\in[n]}|\alpha^i|$.
Invoking Theorem~\ref{thm:avginj}, choose $(z,u)$ such that $[f(z,u,\alpha^i)]_{i\in[n]}$ has entries which are nonzero and distinct.
Invoking Theorem~\ref{thm:realrank}, choose $w$ such that $\psi(w[f(z,u,\alpha^i)]_{i\in[n]}^\top+\eta\one_m\one_n^\top)$ has rank $n$.
Let $Y'\in(\ri\Delta^{\omega-1})^n$. Then there exists $Y\in\R^{\omega\times n}$ such that $\phi(Y)=Y'$ where $\phi$ is applied column-wise. Furthermore, there exists $V\in\R^{m\times \omega}$ such that $Y=V^\top\psi(w[f(z,u,\alpha^i)]_{i\in[n]}^\top+\eta\one_m\one_n^\top)$. So, by defining $Z=\one_dz^\top/\sqrt{d}$ and $W=\one_dw^\top/\sqrt{d}$, we get that $Y=[h_{\theta}(\alpha^i)]_{i\in[n]}$, completing the proof.
\end{proof}

Interestingly, the proof is much simpler than for the transformer model. However, it is important to keep in mind that the memory capacity of a model is only one aspect of it. While the memory capacity is interrelated with the optimization and generalization properties of a model, it does not determine them. Thus, it is important to analyze the models that are actually used in practice, which we have done for transformers. On the other hand, when this analysis suggests alternative models, as it has here, it makes sense to try them in practice. So, we leave it as a future research direction to explore the usefulness of token-averaged FNNs in practice.

\section{Conclusion}
\label{sec:conclusion}

We precisely defined the concept of next-token prediction capacity, proved general upper bounds, and proved a lower bound for one-layer multi-head decoder-only transformers. Assuming that the activation function is continuously differentiable and $(\ran,\exp)$-definable (e.g. arctan, tanh, GELU), we showed that $\Omega(n\omega)$ parameters are necessary to memorize a data set with $n$ unique contexts and a vocabulary of size $\omega$. Assuming that the activation function is real analytic at a point and not a polynomial there (e.g. arctan, tanh, GELU), we showed that $\Omega(n\omega)$ parameters is also sufficient. Furthermore, we provided numerical evidence that the transformer can already be trained to the entropy lower bound in the $\Theta(n\omega)$ parameter regime. Theoretically investigating optimization and generalization in the $\Theta(n\omega)$ parameter regime is left as a future direction. It is also left as a future direction to analyze the next-token prediction capacity for multi-layer transformers and for modified architectures, such as Mixture of Softmaxes, which are not back-loaded, i.e. which have a greater proportion of parameters in the embedding and self-attention sub-layers than in the FNN sub-layer. Finally, it seems worthwhile to consider memory capacity in the more constrained setting of stationary processes.

\bibliographystyle{IEEEtran}
\bibliography{main}

\begin{thebibliography}{10}
\providecommand{\url}[1]{#1}
\csname url@samestyle\endcsname
\providecommand{\newblock}{\relax}
\providecommand{\bibinfo}[2]{#2}
\providecommand{\BIBentrySTDinterwordspacing}{\spaceskip=0pt\relax}
\providecommand{\BIBentryALTinterwordstretchfactor}{4}
\providecommand{\BIBentryALTinterwordspacing}{\spaceskip=\fontdimen2\font plus
\BIBentryALTinterwordstretchfactor\fontdimen3\font minus \fontdimen4\font\relax}
\providecommand{\BIBforeignlanguage}[2]{{%
\expandafter\ifx\csname l@#1\endcsname\relax
\typeout{** WARNING: IEEEtran.bst: No hyphenation pattern has been}%
\typeout{** loaded for the language `#1'. Using the pattern for}%
\typeout{** the default language instead.}%
\else
\language=\csname l@#1\endcsname
\fi
#2}}
\providecommand{\BIBdecl}{\relax}
\BIBdecl

\bibitem{bengio2000neural}
Y.~Bengio, R.~Ducharme, and P.~Vincent, ``A neural probabilistic language model,'' in \emph{Neural Information Processing Systems (NeurIPS)}, vol.~13, 2000.

\bibitem{mikolov2012phd}
T.~Mikolov, ``Statistical language models based on neural networks,'' Ph.D. thesis, Brno University of Technology, Faculty of Information Technology, 2012.

\bibitem{bahdanau2015neural}
D.~Bahdanau, K.~Cho, and Y.~Bengio, ``Neural machine translation by jointly learning to align and translate,'' in \emph{International Conference on Learning Representations (ICLR)}, 2015.

\bibitem{cheng2016long}
J.~Cheng, L.~Dong, and M.~Lapata, ``Long short-term memory-networks for machine reading,'' in \emph{Conference on Empirical Methods in Natural Language Processing (EMNLP)}, 2016, pp. 551--561.

\bibitem{vaswani2017attention}
A.~Vaswani, N.~Shazeer, N.~Parmar, J.~Uszkoreit, L.~Jones, A.~N. Gomez, L.~u. Kaiser, and I.~Polosukhin, ``Attention is all you need,'' in \emph{Neural Information Processing Systems (NeurIPS)}, vol.~30, 2017.

\bibitem{shannon1948mathematical}
C.~E. Shannon, ``A mathematical theory of communication,'' \emph{The Bell System Technical Journal}, vol.~27, pp. 379--423, 623--656, 1948.

\bibitem{shannon1951prediction}
------, ``Prediction and entropy of printed english,'' \emph{The Bell System Technical Journal}, vol.~30, no.~1, pp. 50--64, 1951.

\bibitem{cover1965geometrical}
T.~M. Cover, ``Geometrical and statistical properties of systems of linear inequalities with applications in pattern recognition,'' \emph{IEEE transactions on electronic computers}, vol.~3, pp. 326--334, 1965.

\bibitem{belkin2019reconciling}
M.~Belkin, D.~Hsu, S.~Ma, and S.~Mandal, ``Reconciling modern machine-learning practice and the classical bias--variance trade-off,'' \emph{Proceedings of the National Academy of Sciences}, vol. 116, no.~32, pp. 15\,849--15\,854, 2019.

\bibitem{markov1913}
A.~A. Markov, ``An example of statistical investigation of the text eugene onegin concerning the connection of samples in chains,'' \emph{Science in Context}, vol.~19, no.~4, p. 591–600, 2006.

\bibitem{Billingsley}
P.~Billingsley, \emph{Probability and Measure}, 3rd~ed.\hskip 1em plus 0.5em minus 0.4em\relax John Wiley \& Sons, 1995.

\bibitem{booth1973applying}
T.~L. Booth and R.~A. Thompson, ``Applying probability measures to abstract languages,'' \emph{IEEE Transactions on Computers}, vol. C-22, pp. 442--450, 1973.

\bibitem{du2023measure}
L.~Du, L.~Torroba~Hennigen, T.~Pimentel, C.~Meister, J.~Eisner, and R.~Cotterell, ``A measure-theoretic characterization of tight language models,'' in \emph{Proceedings of the 61st Annual Meeting of the Association for Computational Linguistics (Volume 1: Long Papers)}, 2023, pp. 9744--9770.

\bibitem{gunning1965analytic}
R.~C. Gunning and H.~Rossi, \emph{Analytic functions of several complex variables}.\hskip 1em plus 0.5em minus 0.4em\relax Englewood Cliffs, N.J.: Prentice-Hall, Inc., 1965.

\bibitem{baum1988multilayer}
E.~B. Baum, ``On the capabilities of multilayer perceptrons,'' \emph{Journal of Complexity}, vol.~4, no.~3, pp. 193--215, 1988.

\bibitem{sakurai1992networks}
A.~Sakurai, ``n-h-1 networks store no less n*h+1 examples, but sometimes no more,'' in \emph{International Joint Conference on Neural Networks (IJCNN)}, vol.~3, 1992, pp. 936--941.

\bibitem{huang1991bounds}
S.-C. Huang and Y.-F. Huang, ``Bounds on the number of hidden neurons in multilayer perceptrons,'' \emph{IEEE Transactions on Neural Networks}, vol.~2, no.~1, pp. 47--55, 1991.

\bibitem{yamasaki1993lower}
M.~Yamasaki, ``The lower bound of the capacity for a neural network with multiple hidden layers,'' in \emph{International Conference on Artificial Neural Networks (ICANN)}, 1993, pp. 546--549.

\bibitem{huang2003learning}
G.-B. Huang, ``Learning capability and storage capacity of two-hidden-layer feedforward networks,'' \emph{IEEE Transactions on Neural Networks}, vol.~14, no.~2, pp. 274--281, 2003.

\bibitem{yun2019small}
C.~Yun, S.~Sra, and A.~Jadbabaie, ``Small relu networks are powerful memorizers: a tight analysis of memorization capacity,'' in \emph{Neural Information Processing Systems (NeurIPS)}, vol.~32, 2019.

\bibitem{bubeck2020network}
S.~Bubeck, R.~Eldan, Y.~T. Lee, and D.~Mikulincer, ``Network size and size of the weights in memorization with two-layers neural networks,'' in \emph{Neural Information Processing Systems (NeurIPS)}, vol.~33, 2020, pp. 4977--4986.

\bibitem{madden2024memory}
L.~Madden and C.~Thrampoulidis, ``Memory capacity of two layer neural networks with smooth activations,'' \emph{SIAM Journal on Mathematics of Data Science}, vol.~6, no.~3, pp. 679--702, 2024.

\bibitem{madden2024three}
L.~Madden, ``Interpolation with deep neural networks with non-polynomial activations: necessary and sufficient numbers of neurons,'' \emph{arXiv preprint arXiv:2405.13738}, 2024.

\bibitem{kim2023provable}
J.~Kim, M.~Kim, and B.~Mozafari, ``Provable memorization capacity of transformers,'' in \emph{International Conference on Learning Representations (ICLR)}, 2023.

\bibitem{kajitsuka2024are}
T.~Kajitsuka and I.~Sato, ``Are transformers with one layer self-attention using low-rank weight matrices universal approximators?'' in \emph{Conference on Learning Theory (COLT)}, 2024.

\bibitem{mahdavi2024memorization}
S.~Mahdavi, R.~Liao, and C.~Thrampoulidis, ``Memorization capacity of multi-head attention in transformers,'' in \emph{International Conference on Learning Representations (ICLR)}, 2024.

\bibitem{tian2023scan}
Y.~Tian, Y.~Wang, B.~Chen, and S.~S. Du, ``Scan and snap: Understanding training dynamics and token composition in 1-layer transformer,'' in \emph{Neural Information Processing Systems (NeurIPS)}, 2023.

\bibitem{chen2024provably}
S.~Chen and Y.~Li, ``Provably learning a multi-head attention layer,'' \emph{arXiv preprint arXiv:2402.04084}, 2024.

\bibitem{tarzanagh2023transformers}
D.~A. Tarzanagh, Y.~Li, C.~Thrampoulidis, and S.~Oymak, ``Transformers as support vector machines,'' in \emph{NeurIPS 2023 Workshop on Mathematics of Modern Machine Learning}, 2023.

\bibitem{bartlett-in-context}
R.~Zhang, S.~Frei, and P.~L. Bartlett, ``Trained transformers learn linear models in-context,'' \emph{Journal of Machine Learning Research (JMLR)}, vol.~25, no.~49, pp. 1--55, 2024.

\bibitem{deora2024on}
P.~Deora, R.~Ghaderi, H.~Taheri, and C.~Thrampoulidis, ``On the optimization and generalization of multi-head attention,'' \emph{Transactions on Machine Learning Research (TMLR)}, 2024.

\bibitem{makkuva2024attention}
A.~V. Makkuva, M.~Bondaschi, A.~Girish, A.~Nagle, M.~Jaggi, H.~Kim, and M.~Gastpar, ``Attention with markov: A framework for principled analysis of transformers via markov chains,'' \emph{arXiv preprint arXiv:2402.04161}, 2024.

\bibitem{li2024mechanics}
Y.~Li, Y.~Huang, M.~E. Ildiz, A.~S. Rawat, and S.~Oymak, ``Mechanics of next token prediction with self-attention,'' in \emph{International Conference on Artificial Intelligence and Statistics (AISTATS)}, 2024.

\bibitem{tian2023joma}
Y.~Tian, Y.~Wang, Z.~Zhang, B.~Chen, and S.~S. Du, ``Jo{MA}: Demystifying multilayer transformers via joint dynamics of {MLP} and attention,'' in \emph{International Conference on Learning Representations (ICLR)}, 2024.

\bibitem{ildiz2024self}
M.~E. Ildiz, Y.~Huang, Y.~Li, A.~S. Rawat, and S.~Oymak, ``From self-attention to markov models: Unveiling the dynamics of generative transformers,'' \emph{arXiv preprint arXiv:2402.13512}, 2024.

\bibitem{malach2023auto}
E.~Malach, ``Auto-regressive next-token predictors are universal learners,'' \emph{arXiv preprint arXiv:2309.06979}, 2023.

\bibitem{thrampoulidis2024implicit}
C.~Thrampoulidis, ``Implicit bias of next-token prediction,'' \emph{arXiv preprint arXiv:2402.18551}, 2024.

\bibitem{sanford2023represent}
C.~Sanford, D.~J. Hsu, and M.~Telgarsky, ``Representational strengths and limitations of transformers,'' in \emph{Neural Information Processing Systems (NeurIPS)}, vol.~36, 2023, pp. 36\,677--36\,707.

\bibitem{wang2024transformers}
Z.~Wang, S.~Wei, D.~Hsu, and J.~D. Lee, ``Transformers provably learn sparse token selection while fully-connected nets cannot,'' in \emph{International Conference on Machine Learning (ICML)}, vol. 235, 2024, pp. 51\,854--51\,912.

\bibitem{sontag1997shattering}
E.~D. Sontag, ``Shattering all sets of k points in ``general position'' requires (k - 1)/2 parameters,'' \emph{Neural Computation}, vol.~9, no.~2, pp. 337--348, 1997.

\bibitem{tamura1991capabilities}
S.~Tamura, ``Capabilities of a three layer feedforward neural network,'' \emph{IEEE International Joint Conference on Neural Networks}, vol.~3, pp. 2757--2762, 1991.

\bibitem{tamura1997capabilities}
S.~Tamura and M.~Tateishi, ``Capabilities of a four-layered feedforward neural network: four layers versus three,'' \emph{IEEE Transactions on Neural Networks}, vol.~8, no.~2, pp. 251--255, 1997.

\bibitem{huang1998upper}
G.-B. Huang and H.~Babri, ``Upper bounds on the number of hidden neurons in feedforward networks with arbitrary bounded nonlinear activation functions,'' \emph{IEEE Transactions on Neural Networks}, vol.~9, no.~1, pp. 224--229, 1998.

\bibitem{sard1941measure}
A.~Sard, ``The measure of the critical values of differentiable maps,'' \emph{Bulletin of the American Mathematical Society}, vol.~48, pp. 883--890, 1942.

\bibitem{dries1994elementary}
L.~van~den Dries, A.~Macintyre, and D.~Marker, ``The elementary theory of restricted analytic fields with exponentiation,'' \emph{Annals of Mathematics}, vol. 140, pp. 183--205, 1994.

\bibitem{dries1994real}
L.~van~den Dries and C.~Miller, ``On the real exponential field with restricted analytic functions,'' \emph{Israel Journal of Mathematics}, vol.~85, pp. 19--56, 1994.

\bibitem{dries1998tame}
L.~van~den Dries, \emph{Tame Topology and O-minimal Structures}.\hskip 1em plus 0.5em minus 0.4em\relax Cambridge University Press, 1998.

\bibitem{sontag1996critical}
E.~D. Sontag, ``Critical points for least-squares problems involving certain analytic functions, with applications to sigmoidal nets,'' \emph{Advances in Computational Mathematics}, vol.~5, no.~1, pp. 245--268, 1996.

\bibitem{kruskal1977three}
J.~B. Kruskal, ``Three-way arrays: rank and uniqueness of trilinear decompositions, with application to arithmetic complexity and statistics,'' \emph{Linear Algebra and its Applications}, vol.~18, no.~2, pp. 95--138, 1977.

\bibitem{gantmacher1960matrices}
F.~R. Gantmacher, \emph{The Theory of Matrices, Volume 1}.\hskip 1em plus 0.5em minus 0.4em\relax New York: Chelsea Publishing Company, 1960.

\bibitem{axler2015linear}
S.~Axler, \emph{Linear Algebra Done Right, Third Edition}.\hskip 1em plus 0.5em minus 0.4em\relax New York: Springer, 2015.

\bibitem{hermite1873}
Hermite, ``Sur la fonction exponentielle,'' \emph{Comptes rendus de l'Académie des Sciences}, 1873.

\bibitem{Fréchet}
M.~Fréchet, ``Des familles et fonctions additivies d'ensembles abstraits,'' \emph{Fundamenta Mathematicae}, vol.~5, pp. 206--251, 1924.

\bibitem{gage1994new}
P.~Gage, ``A new algorithm for data compression,'' \emph{The C Users Journal}, vol.~12, no.~2, p. 23–38, 1994.

\bibitem{shai2}
S.~Shalev-Shwartz and S.~Ben-David, \emph{Understanding Machine Learning: From Theory to Algorithms}.\hskip 1em plus 0.5em minus 0.4em\relax Cambridge University Press, 2014.

\bibitem{eldan2023tinystories}
R.~Eldan and Y.~Li, ``Tinystories: How small can language models be and still speak coherent english?'' \emph{arXiv preprint arXiv:2305.07759}, 2023.

\bibitem{kingma2015adam}
D.~P. Kingma and J.~Ba, ``Adam: {A} method for stochastic optimization,'' in \emph{International Conference on Learning Representations (ICLR)}, 2015.

\end{thebibliography}

\end{document}